%% file: main.tex
\documentclass[11pt]{article}
\usepackage[sort&compress,square,comma,authoryear]{natbib}
\usepackage[a4paper, total={6.1in, 10in}]{geometry}
\usepackage{enumitem}
\usepackage[utf8]{inputenc} 
\usepackage[T1]{fontenc}    

\input{math_commands.tex}

\usepackage{hyperref}       
\usepackage{url}            
\usepackage{booktabs}       
\usepackage{dirtytalk}
\usepackage{tabularx}
\usepackage{multirow}
\newcolumntype{Y}{>{\centering\arraybackslash}X}
\newcolumntype{Z}{>{\raggedleft\arraybackslash}X}
\usepackage[parfill]{parskip}
\usepackage{amsfonts}       
\usepackage{nicefrac}       
\usepackage{microtype}      
\usepackage[dvipsnames]{xcolor}         
\usepackage{amsmath,amssymb,amsthm}
\usepackage{algorithm}
\usepackage{bbold}
\usepackage{colortbl}
\usepackage{graphicx} 
\usepackage{mathrsfs,mathtools}
\usepackage{thm-restate}
\usepackage{wrapfig}
\usepackage{caption}
\usepackage[capitalize]{cleveref}
\usepackage{tikz}
\usetikzlibrary{calc, fadings, decorations, shapes, positioning, arrows}
\usepackage[graphicx]{realboxes}
\usepackage{pifont}
\usepackage{tcolorbox}
\usepackage{hyperref}

\usepackage{algpseudocode}
\usepackage{hyperref}   
\usepackage{soul}       
\usepackage{wrapfig}    
\usepackage{caption}
\usepackage[dvipsnames]{xcolor}

\usepackage{subcaption}

\definecolor{salmon}{RGB}{234,153,153}
\definecolor{cornflowerblue}{RGB}{100,149,237}
\hypersetup{
    colorlinks,
    linkcolor={cornflowerblue},
    citecolor={cornflowerblue},
    urlcolor={salmon}
}

\definecolor{darkgreen}{rgb}{0.0, 0.5, 0.0}
\usepackage[textsize=tiny]{todonotes}
\usepackage{tabularx}
\usepackage{wrapfig}
\definecolor{darkblue}{rgb}{0, 0, 0.5}
\hypersetup{colorlinks=true, citecolor=darkblue, linkcolor=darkblue, urlcolor=darkblue}
\usepackage{tikz}

\theoremstyle{plain}
\newtheorem{theorem}{Theorem}[section]

\newtheorem{lemma}[theorem]{Lemma}

\theoremstyle{definition}

\theoremstyle{remark}

\newcommand{\midsepremove}{\aboverulesep = 0mm \belowrulesep = 0mm}
\midsepremove


\definecolor{redborder}{HTML}{A74B2A}

\newcommand{\imagepath}{imgs/observations/}

\newcommand{\advcircle}{_circle_thick}

\title{Pruning Cannot Hurt Robustness: \\Certified Trade-offs in Reinforcement Learning}

\date{}

\usepackage{authblk}

\author{%
James Pedley\thanks{Correspondence to jpedley@robots.ox.ac.uk} \quad Benjamin Etheridge \quad Stephen J. Roberts \quad Francesco Quinzan 
}

\affil{Machine Learning Research Group\\
Department of Engineering Science\\University of Oxford}

\begin{document}

\maketitle


\input{sections/0_abstract}
%


\input{sections/1_intro}

\input{sections/2_literature}

\input{sections/3_theory}

\input{sections/4_method}


\input{sections/6_results}

\input{sections/7_conclusion}

\bibliographystyle{apalike}
\bibliography{references}
\newpage
\renewcommand{\thesection}{\Alph{section}}
\setcounter{section}{0}
\noindent {\LARGE\textbf{Appendices}}
%
\input{sections/x_appendices}

\end{document}

%% file: math_commands.tex
\usepackage{amsmath,amsfonts,bm}

\def\eqref#1{equation~\ref{#1}}

\def\1{\bm{1}}

\DeclareMathAlphabet{\mathsfit}{\encodingdefault}{\sfdefault}{m}{sl}
\SetMathAlphabet{\mathsfit}{bold}{\encodingdefault}{\sfdefault}{bx}{n}

\newcommand{\E}{\mathbb{E}}



%% file: sections/0_abstract.tex
\begin{abstract}
\noindent
Reinforcement learning (RL) policies deployed in real-world environments must remain reliable under adversarial perturbations. At the same time, modern deep RL agents are heavily overparameterized, raising costs and fragility concerns. While pruning has been shown to improve robustness in supervised learning, its role in adversarial RL remains poorly understood. We develop the first theoretical framework for \emph{certified robustness under pruning} in state-adversarial Markov decision processes (SA-MDPs). For Gaussian and categorical policies with Lipschitz networks, we prove that elementwise pruning can only tighten certified robustness bounds; pruning never makes the policy less robust. Building on this, we derive a novel three-term regret decomposition that disentangles clean-task performance, pruning-induced performance loss, and robustness gains, exposing a fundamental performance--robustness frontier. Empirically, we evaluate magnitude and micro-pruning schedules on continuous-control benchmarks with strong policy-aware adversaries. Across tasks, pruning consistently uncovers reproducible ``sweet spots'' at moderate sparsity levels, where robustness improves substantially without harming---and sometimes even enhancing---clean performance. These results position pruning not merely as a compression tool but as a structural intervention for robust RL.
\end{abstract}

%% file: sections/1_intro.tex
\section{Introduction}
\label{sec:intro}

Reinforcement learning (RL) has demonstrated impressive capabilities in domains ranging from
strategic games \citep{DBLP:journals/nature/SilverSSAHGHBLB17} to robotic control
\citep{DBLP:journals/corr/LillicrapHPHETS15}. RL is now employed in various safety-critical
applications, such as for autonomous vehicles \citep{DBLP:conf/icra/KendallHJMRALBS19},
computer network defence \citep{DBLP:journals/corr/abs-2306-09318}, and language model
alignment \citep{DBLP:conf/nips/Ouyang0JAWMZASR22}, often without human-in-the-loop
supervision. It is therefore of crucial importance to consider how robust RL policies are against
malicious actors who would seek to adversarially manipulate their actions, and how we might better
defend against such attacks.

Modern model-free RL policies are typically over-parameterized
\citep{DBLP:conf/icml/SokarACE23, DBLP:conf/nips/Thomas22}, which makes them more expensive
to deploy and fragile to distribution shift \citep{DBLP:conf/iclr/KumarA0CTL22,
DBLP:conf/iclr/MenonRK21}. A natural solution is pruning, which has been widely explored in
supervised learning for model compression \citep{DBLP:conf/iclr/HayouTDT21}, improved
generalization \citep{DBLP:conf/nips/CunDS89}, and robustness to adversarial attacks
\citep{DBLP:conf/nips/Sehwag0MJ20, DBLP:journals/tmlr/LiCLLW23}. However, unlike in supervised learning, the relationship between pruning and robustness in RL remains largely unexplored \citep{DBLP:conf/icml/GraesserEEC22, DBLP:conf/nips/0001CX0LBH20}. Reinforcement learning poses unique challenges: perturbations to observations can propagate and accumulate over long-horizon trajectories, where even small errors may compound into catastrophic failures \citep{DBLP:conf/iclr/WengDUXGSK20}.

In this work, we study \emph{sparse RL robustness} to examine how pruning influences both benign performance and adversarial robustness, and how these often competing objectives can be better aligned. We model adversaries which perturb agent observations through a state adversarial Markov
decision process (Figure~\ref{fig:threat}), building off work from \citet{DBLP:conf/nips/0001CX0LBH20}.
This can be used to show that pruning offers theoretical guarantees, proving that
\textbf{element-wise pruning cannot worsen the bounds of certified robustness}. We derive a
three-term regret decomposition that disentangles clean performance, pruning-induced performance
loss, and robustness gain.

\begin{wrapfigure}{R}{0.5\textwidth}
  \vspace{-2em}
  \centering
  \includegraphics[width=0.98\linewidth]{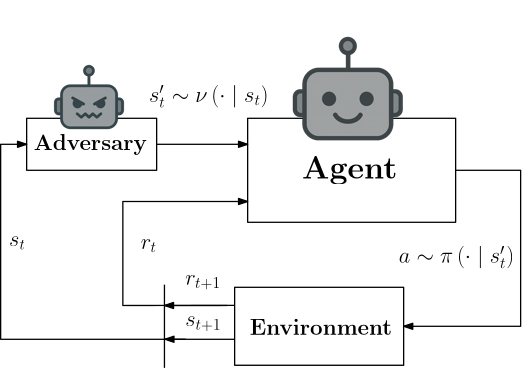}
  \caption{The SA-MDP framework. A victim agent receives a perturbed observation from an adversary trained to reduce its performance. The true state $s_t$ is emitted by the environment, tampered by the adversary, and then passed to the victim.}
  \label{fig:threat}
\end{wrapfigure}
We validate these predictions experimentally using Proximal Policy Optimisation (PPO)
\citep{DBLP:journals/corr/SchulmanWDRK17} across multiple continuous control environments under
a range of strong policy-aware adversaries. Across settings, pruning consistently uncovers a
``sweet spot'' at intermediate sparsities where robustness improves substantially without sacrificing
--- and sometimes even enhancing --- clean performance. We combine pruning with
state-adversarial regularisation \citep{DBLP:conf/nips/0001CX0LBH20} to highlight its effectiveness
as a complementary technique to existing robustness measures. Across three MuJoCo benchmarks,
pruning achieves up to $25\%$ higher certified robustness while maintaining at least $95\%$ of
baseline clean performance, consistently revealing reproducible Pareto optima.

\noindent\textbf{Our contributions.}\\
\begin{itemize}
\item 
\noindent \textit{Theoretical guarantees:} We prove that pruning monotonically improves certified robustness in SA-MDPs, establishing that sparsity cannot reduce adversarial resilience.  
\item
\noindent \textit{Trade-off characterization:} We derive a three-term regret bound that formalizes the interplay between pruning, clean performance, and robustness, clarifying when these align or conflict.  
\item
\noindent \textit{Empirical validation:} We show, across continuous-control benchmarks, that pruning consistently uncovers reproducible ``sweet spots'' where robustness gains outweigh performance losses.  
\end{itemize}

%% file: sections/2_literature.tex
\section{Related work}
\label{sec:lit}

\subsection{Adversarial attacks and robustness in RL}
\label{sec:lit:adv_rl}
%
\textbf{Training-phase attacks.}\label{sec:lit:adv_rl:training} A first class of training-phase attacks are reward attacks. As rewards formally characterize an agent’s purpose, altering the rewards logically changes the learned policies of the agents. A reward-poisoning attack was proposed in batch RL \citep{DBLP:conf/aaai/ZhangP08, DBLP:conf/sigecom/ZhangPC09}, where rewards were stored in an unlocked, pre-collected dataset. This provided the attacker with the opportunity to directly change the reward in the dataset. Variants of this attack have also been proposed \citep{DBLP:conf/gamesec/HuangZ19a,DBLP:conf/icml/RakhshaRD0S20,DBLP:journals/corr/abs-2102-08492,DBLP:journals/access/CaiZH22}. These variants use different oracle access to the model being attacked.
It is further possible to attack an RL agent, without tampering the reward. For example, \cite{DBLP:conf/atal/XuWRR21,DBLP:conf/atal/Xu22} propose an environment-poisoning attack, where the victim RL agent is misled by subtle changes to the environment.
RL agents can also be attacked by embedding triggers that elicit malicious behaviour \citep{DBLP:conf/dac/KiourtiWJL20,DBLP:conf/globecom/Yu0LHF22}. Here, the attacker alters the training process so that the agent learns to associate a rare pattern (the trigger) with an attacker-chosen behaviour.
We remark that training-phase attacks have also been studied for RL from Human Feedback (RLHF), mostly in the context of LLMs fine-tuning, e.g., \citep{DBLP:conf/acl/Wang0CVX24,DBLP:journals/corr/abs-2409-00787,DBLP:journals/corr/abs-2304-12298,DBLP:conf/iclr/RandoT24,DBLP:journals/corr/abs-2406-06852}.

\textbf{Test-phase attacks.}\label{sec:lit:adv_rl:test} Test-phase attacks aim to deceive a trained policy. One common approach involves introducing perturbations into the state space at different points during execution \citep{DBLP:conf/iclr/HuangPGDA17,DBLP:conf/iclr/LinHLS0S17,DBLP:conf/iclr/KosS17}. Beyond this, carefully crafted perturbation sequences can be designed to steer agents toward specific states \cite{DBLP:conf/mldm/BehzadanM17,DBLP:conf/iclr/LinHLS0S17,DBLP:journals/corr/abs-1805-12487,DBLP:conf/iclr/WengDUXGSK20,DBLP:conf/atal/HussenotGP20,DBLP:journals/tdsc/MoTLY23}. Such perturbation-based attacks, however, can be mitigated using techniques that reinforce cumulative rewards \cite{DBLP:conf/ccs/ChanWY20}. In addition, research has explored test-phase transferability attacks \cite{DBLP:conf/iclr/HuangPGDA17,DBLP:conf/icassp/YangQCOHLM20,DBLP:conf/atal/InkawhichCL20}, which exploit the empirical observation that adversarial examples crafted to deceive one model (a surrogate) can also mislead other models, even when those models differ in architecture, training data, or parameters. More recently, test-phase attacks have also been devised to specifically target RLHF in LLMs \citep{DBLP:journals/corr/abs-2309-11166, DBLP:journals/chinaf/XiCGHDHZWJZZFWXZWJZLYDW25, DBLP:conf/mm/LiuC0Y024, DBLP:journals/corr/abs-2306-04528}. For further information we refer interested readers to surveys such as \cite{DBLP:journals/csur/DasAW25,DBLP:journals/corr/abs-2312-02003,DBLP:journals/corr/abs-2310-10844}.

\textbf{Robustness in RL.}\label{sec:lit:adv_rl:robustness} While robustness in RL has been extensively explored, studies specifically addressing adversarial robustness remain limited. Empirical robust learning typically uses heuristics or evaluations to enhance model reliability. An effective way of improving robustness is to use Adversarially Robust Policy Learning (ARPL), which incorporates physically plausible adversarial examples during training \citep{DBLP:conf/iros/MandlekarZGFS17,DBLP:conf/icml/TesslerEM19,DBLP:journals/tnn/ZhouLZ24}. It is likewise possible to make agents more resilient, by altering the environment during training \cite{DBLP:conf/nips/JiangDPFGR21,DBLP:conf/icml/Parker-HolderJ022,DBLP:conf/nips/0001JVBRCL20}.
Additional contributions include \cite{DBLP:conf/icml/BallLPR21}, showing that Augmented World Models improve generalization and \cite{DBLP:conf/icml/BallPPCR20}, where agents use a context variable to adapt to changes in environment dynamics. 
\subsection{Pruning}
\label{sec:lit:pruning} 
Sparsity is valuable not only for model compression and faster inference \citep{DBLP:journals/corr/HanPTD15,DBLP:conf/iclr/MolchanovTKAK17}, but also for improving generalisation \citep{DBLP:conf/nips/CunDS89, DBLP:conf/nips/HassibiS92,DBLP:conf/nips/BartoldsonMBE20}. Pruning design choices include whether to remove individual parameters \citep{DBLP:conf/nips/CunDS89, DBLP:journals/corr/HanPTD15} or use structured sparsity \citep{DBLP:conf/nips/WenWWCL16, DBLP:conf/iclr/LasbyGENI24}, and whether to prune statically \citep{DBLP:conf/iclr/FrankleC19} or dynamically during training \citep{DBLP:conf/icml/EvciGMCE20, mocanu_scalable_2018, DBLP:journals/ijon/Prechelt97}. Criteria include random selection \citep{DBLP:conf/iclr/LiuCC0MWP22}, magnitude \citep{DBLP:conf/nips/CunDS89}, saliency \citep{DBLP:conf/nips/HassibiS92}, or evolutionary strategies \citep{mocanu_scalable_2018}, often paired with Straight-Through Estimators \citep{DBLP:conf/wacv/VanderschuerenV23, DBLP:journals/corr/BengioLC13, Hinton2012Coursera} for gradient flow through binary masks. For a detailed overview, see, e.g., \cite{DBLP:journals/pami/ChengZS24a}.

For supervised learning, it has been empirically demonstrated that pruning can improve robustness against adversarial attacks, both through the above methods and augmenting with additional adversarial objectives. In \cite{DBLP:journals/corr/abs-1912-02386} lottery tickets \citep{DBLP:conf/iclr/FrankleC19}, pruned up to $\sim96\%$, can outperform the original network on adversarial accuracy. Work in \cite{DBLP:conf/nips/FuYZWOCL21} extends \cite{DBLP:conf/iclr/FrankleC19, DBLP:conf/icml/MalachYSS20} to show that tickets exist which can outperform the dense network on adversarial examples, without any training. HYRDA \citep{DBLP:conf/nips/Sehwag0MJ20} creates a risk minimisation objective for pruning which optimises the pruning to be adversarially robust. Conversely, \cite{DBLP:journals/corr/abs-1912-02386} separately applies pruning followed by adversarial training \citep{DBLP:conf/iclr/MadryMSTV18} to produce more robust sparse networks and \cite{DBLP:conf/iclr/BairYS0024} introduces a sharpness-aware pruning criterion to encourage flatter, more generalisable networks. In contrast, far less work has been done to understand the interaction of sparsity and robustness for RL policies, motivating this work.

%% file: sections/3_theory.tex
\section{Robustness Bounds in SA-MDPs} \label{sec:theory}
%
\subsection{Setting}
We study a state-adversarial Markov decision process (SA-MDP) defined with perturbation sets $B(s)\subseteq\mathcal{S}$. A standard MDP is specified as a tuple $(\mathcal{S},\mathcal{A},R,p,\gamma)$, where a stationary stochastic policy is given by $\pi_\theta:\mathcal{S}\to\mathcal{P}(\mathcal{A})$ with density $\pi_\theta(a|s)$. In the SA-MDP setting, the agent does not act on the true state $s$ but instead observes an adversarially perturbed state $\nu(s)\in B(s)$ and selects actions according to $\pi_\theta(\cdot|\nu(s))$, while the environment transitions based on the true state through $p(\cdot|s,a)$. Consequently, an SA-MDP can be represented as $(\mathcal{S},\mathcal{A},B,R,p,\gamma)$. in this work, we constrain $\nu$ to an $\ell_p$ ball: $B(s)\!:=\!\{\hat s\in\mathcal{S}:\|\hat s-s\|_p\le\varepsilon\}$ with budget $\varepsilon\!>\!0$ and $p\!\in\!\{2,\infty\}$.

For distributions $P,Q$ on $\mathcal{A}$, we define the total variation distance as 
$$D_{\mathrm{TV}}(P,Q):=\sup_{E\subseteq\mathcal{A}}|P(E)-Q(E)|.$$
For each state $s\in \mathcal S$, we define 
$TV_{\max}(s;\theta):=\max_{\hat s\in B(s)}D_{\mathrm{TV}}(\pi_\theta(\cdot|s),\pi_\theta(\cdot|\hat s))$. 
Let $d_\mu^{\pi_\theta}$ be the discounted visitation distribution from $\mu$, and set
\[
F(\theta):=\mathbb{E}_{s\sim d_\mu^{\pi_\theta}}[TV_{\max}(s;\theta)],\qquad
\mathcal B(\theta):=\alpha\,F(\theta)\ \ \text{with} \ \
\alpha=2\!\left[1+\tfrac{\gamma}{(1-\gamma)^2}\right]R_{\max},
\]
with $|R(s,a,s')|\le R_{\max}$.

\textbf{Policy classes and constants.} 
We consider (i) Gaussian policies $\pi_\theta(a|s)=\mathcal N(\mu_\theta(s),\Sigma)$ with fixed $\Sigma\succ0$, and 
(ii) categorical policies $\pi_\theta(\cdot|s)=\mathrm{softmax}(z_\theta(s))$, $z_\theta(s)\in\mathbb{R}^K$. 
In subsequent bounds we use the constant 
$c=(\sqrt{2\pi\,\lambda_{\min}(\Sigma)}) ^{-1}$ for Gaussian policies and 
$c=1/4$ for categorical softmax policies. We write $\tilde V^{\pi_\theta\circ\nu^*}$ for the robust value under the optimal adversary. 
The robustness gap for a state $s\in\mathcal S$ is $V^{\pi_{\theta'}}(s)-\tilde V^{\pi_{\theta'}\circ\nu^*(\pi_{\theta'})}(s)$.

\textbf{Additional notation.} We use $\|x\|_p$ for vector $\ell_p$ norms; $\|W\|_F$ (Frobenius), $\|W\|_1\!=\!\max_j\sum_i|W_{ij}|$, $\|W\|_\infty\!=\!\max_i\sum_j|W_{ij}|$ for matrices; and $\|J\|_{\mathrm{op}}$ for spectral norm. For neural policies, $g_\theta(s)$ denotes logits/means and $J_{g_\theta}(s)$ its Jacobian.
%

\subsection{Performance–robustness trade-offs}
We show that elementwise pruning of a policy network in stochastic action MDPs cannot worsen its certified robustness guarantee. This result follows from a surrogate Lipschitz bound, which decreases under pruning, thereby ensuring monotone improvement in robustness.
\begin{theorem}[SA-MDP robustness improves under pruning]
\label{thm:prune-certified-robust}
Let $\pi_\theta$ be either a Gaussian or categorical policy realized by a feedforward network with Lipschitz activations $\sigma_\ell$ and weights $\theta$.  
Define the surrogate Lipschitz bound
\[
\tilde L_\theta
:=\Bigg(\prod_{\ell=1}^{L-1}L_{\sigma_\ell}\Bigg)
  \prod_{\ell=1}^L \min\!\Big\{\,\|W_\ell\|_F,\;\sqrt{\|W_\ell\|_1\|W_\ell\|_\infty}\,\Big\}.
\]
Let $\theta'$ be obtained from $\theta$ by elementwise pruning. Then
\[
\max_s \{V^{\pi_\theta}(s)-\tilde V^{\pi_\theta\circ\nu^*(\pi_\theta)}(s)\}
\;\le\;\alpha\,c\,\tilde L_\theta\,\varepsilon,
\]
and
\[
\max_s \{V^{\pi_{\theta'}}(s)-\tilde V^{\pi_{\theta'}\circ\nu^*(\pi_{\theta'})}(s)\}
\;\le\;\alpha\,c\,\tilde L_{\theta'}\,\varepsilon
\;\le\;\alpha\,c\,\tilde L_\theta\,\varepsilon,
\]

Thus, under pure elementwise pruning, the certified robustness bound is monotone nonincreasing.
\end{theorem}
Intuitively, this theorem shows that pruning reduces the network’s sensitivity to perturbations, so the certified robustness of the policy can only improve as parameters are removed.

\noindent\textbf{Training remark.}
The monotonicity result (Theorem~1) applies to pruning on a fixed set of weights. 
During training, gradient steps may enlarge weight norms and hence $\tilde L_\theta$, so robustness is not globally monotone. 
Nevertheless, each pruning step strictly decreases $\tilde L_\theta$ relative to the current parameters, acting as a monotone regularizer counteracting weight growth. 
This explains why robustness tends to improve steadily in practice (Sec.~5) when pruning is interleaved with training.

\noindent\textbf{Tightness of the bound.}
While Theorem~\ref{thm:prune-certified-robust} provides a provably monotone
global robustness bound, it can be loose in practice. 
A sharper, distribution--dependent refinement is given in 
Lemma~\ref{lem:local-bound} (Appendix), which often yields much tighter estimates, 
though without the same monotonicity guarantee under pruning.

Theorem~\ref{thm:prune-certified-robust} guarantees pruning cannot worsen the worst-case robustness gap.
However, worst-case bounds can be overly pessimistic. 
To obtain guarantees that better capture typical performance, we next consider expected versions of the robustness gap, aligned with the population-level objective $F(\theta)$.  
This motivates our second main result, which characterizes robustness through $F(\theta)$ and bounds the expected degradation in value under the optimal adversary.

\begin{theorem}[Unified regret under SA attack]
\label{thm:unified-regret}
Fix a start distribution $\mu$. Write $J(\pi):=\mathbb{E}_{s_0\sim\mu}[V^\pi(s_0)]$ and $\tilde J(\pi):=\mathbb{E}_{s_0\sim\mu}[\tilde V^\pi(s_0)]$. Let $\bar\pi$ be any comparator policy (e.g., $\pi^*$ maximizing $J$ or $\tilde\pi^*$ maximizing $\tilde J$), and let $\nu^*$ denote the optimal SA adversary for $\pi_\theta$. Define
\[
\mathrm{Reg}_{\mathrm{clean}}(\theta;\bar\pi):=J(\bar\pi)-J(\pi_\theta),
\qquad
\mathrm{Reg}_{\mathrm{atk}}(\theta;\bar\pi):=J(\bar\pi)-\tilde J(\pi_\theta\!\circ\!\nu^*).
\]
Then, it holds 
$$\mathrm{Reg}_{\mathrm{atk}}(\theta;\bar\pi)-\mathrm{Reg}_{\mathrm{clean}}(\theta;\bar\pi)
=J(\pi_\theta)-\tilde J(\pi_\theta\!\circ\!\nu^*)
\;\le\;\mathcal B(\theta).$$ Additionally, if $\pi_\theta$ is Gaussian with fixed $\Sigma\!\succ\!0$ or categorical softmax with Lipschitz network, then \[
\mathrm{Reg}_{\mathrm{atk}}(\theta;\bar\pi)-\mathrm{Reg}_{\mathrm{clean}}(\theta;\bar\pi)
\;\le\;\alpha\,c\,\tilde L_\theta\,\varepsilon.
\]
Moreover, if $\theta'$ is obtained by entrywise pruning, then
\[
\mathrm{Reg}_{\mathrm{atk}}(\theta';\bar\pi)-\mathrm{Reg}_{\mathrm{clean}}(\theta';\bar\pi)
\;\le\;\alpha\,c\,\tilde L_{\theta'}\,\varepsilon
\;\le\;\alpha\,c\,\tilde L_\theta\,\varepsilon.
\]
\end{theorem}
Theorem~\ref{thm:unified-regret} shows that the extra regret a policy suffers under the optimal state-adversarial attack (compared to its clean regret) is always bounded by a robustness coefficient $\mathcal B(\theta)$, and in particular by the Lipschitz surrogate $\tilde L_\theta$ for Gaussian or softmax policies. 
In other words, pruning cannot increase this attack–clean regret gap and in fact makes the bound tighter.

\noindent\textbf{Pruning sensitivity.}
For pruned parameters $\theta'=\theta-\Delta\theta$, define the path–averaged parameter sensitivity
\[
\mathcal L_{\mathrm{par}}(\theta,\theta') := \int_0^1 
\Big(\E_{s\sim d_\mu^{\pi_\theta}}\|J_\phi g_\phi(s)\|_{\mathrm{op}}^2\Big)^{1/2}_{\phi=\theta'+t(\theta-\theta')}\mathrm{d}t,
\]
with $g_\phi=\mu_\phi$, for Gaussian policies and
$g_\phi=z_\phi$ for categorical.  
Lemma~\ref{thm:value-drop} (Appendix) shows that both clean and attacked value drops satisfy
\[
J(\pi_\theta)-J(\pi_{\theta'}) ,\;\;
\tilde J(\pi_\theta\!\circ\!\nu^*)-\tilde J(\pi_{\theta'}\!\circ\!\nu^*)
\;\le\; \alpha c\,\mathcal L_{\mathrm{par}}(\theta,\theta')\,\|\Delta\theta\|.
\]
This term serves as the performance loss from pruning in Theorem~\ref{thm:three-term}, complementing the baseline regret and robustness gap to yield the full three–term trade-off.

\begin{theorem}[Performance--robustness trade-off under pruning]
\label{thm:three-term}
Fix any comparator policy $\bar\pi$. For pruned parameters $\theta'=\theta-\Delta\theta$,
\[
\mathrm{Reg}_{\mathrm{atk}}(\theta';\bar\pi)
\;\le\;
\underbrace{\mathrm{Reg}_{\mathrm{clean}}(\theta;\bar\pi)}_{\text{clean regret of unpruned}}
+\underbrace{\alpha\,c\,\mathcal L_{\mathrm{par}}(\theta,\theta')\,\|\Delta\theta\|}_{\text{performance loss from pruning}}
+\underbrace{\alpha\,c\,\tilde L_{\theta'}\,\varepsilon}_{\text{robustness gap of pruned}}.
\]
\end{theorem}

\noindent\textbf{Interpretation.}
The three--term bound exposes a fundamental performance--robustness trade--off. 
The first term is the clean regret of the unpruned policy, determined by baseline training quality. 
The second term, $\alpha c\,\mathcal L_{\mathrm{par}}(\theta,\theta')\|\Delta\theta\|$, is the performance loss from pruning. 
Here $\mathcal L_{\mathrm{par}}$ is a path--averaged sensitivity: it measures how strongly the policy’s outputs react to parameter perturbations along the path from $\theta$ to $\theta'$. 
Low sensitivity implies pruning has little effect, while high sensitivity makes small weight changes costly. 
This explains why magnitude pruning is effective: removing small--magnitude weights keeps $\|\Delta\theta\|$ small, reducing the penalty.

The third term, $\alpha c\,\tilde L_{\theta'}\varepsilon$, is the robustness gap of the pruned policy, controlled by its input Lipschitz constant. 
Because pruning reduces $\tilde L_\theta$, this term always improves. 
Thus pruning simultaneously hurts via performance loss and helps via robustness, and the optimal sparsity balances these opposing effects. 
Pruning therefore acts not just as compression but as a structural intervention trading margin for robustness.

%% file: sections/4_method.tex
\section{Experiments and Results}
\label{sec:method}

We study the performance--robustness trade-off predicted by our SA-MDP theory under structured network sparsification. Concretely, we couple on-the-fly weight pruning with adversarially robust policy optimization on continuous-control benchmarks. This section specifies environments, policies, attacks, pruning strategies, and the full training--evaluation protocol. 

\subsection{Tasks and Policies}
\label{sec:tasks}
We evaluate on three MuJoCo continuous-control tasks from Gym: \texttt{Hopper}, \texttt{Walker2d}, and \texttt{HalfCheetah}. Policies are stochastic Gaussian actors $\pi_\theta(a\mid s)=\mathcal N(\mu_\theta(s), \Sigma)$ with state-independent diagonal covariance $\Sigma$. Value functions use separate MLPs. Unless otherwise noted, both actor and critic are multilayer perceptrons (MLPs) with Lipschitz activations and standard initialization (full architecture details are provided in Appendix~\ref{app:config}).

\subsection{State-Adversarial Training Objective}
\label{sec:sa-objective}
Our theory (Sec.~\ref{sec:theory}) shows that robustness bounds are governed by divergences between $\pi_\theta(\cdot\mid s)$ and $\pi_\theta(\cdot\mid \hat s)$ for perturbed states $\hat s\!\in\!B(s)$. By Pinsker’s inequality, these total variation terms can be controlled by KL divergences. To operationalize this, we adopt the SA-regularization term introduced in prior work on robust PPO (\citet{DBLP:conf/nips/0001CX0LBH20}):
\[
\mathcal R_{\mathrm{SA}}(\theta)\;=\;\mathbb{E}_{s\sim d_\mu^{\pi_\theta}}\Big[\max_{\hat s\in B(s)} D_{KL}\big(\pi_\theta(\cdot\mid s)\,\|\,\pi_\theta(\cdot\mid \hat s)\big)\Big],
\]
where $D_{KL}$ is instantiated as KL divergence. While KL does not appear directly in the robustness bounds, it serves as a theoretically justified surrogate via Pinsker’s inequality, and has been widely used in the literature on adversarially robust RL.  

The actor objective becomes
\(
\mathcal L_{\pi}(\theta)\;=\;\mathcal L_{\text{PPO}}(\theta)\;+\;\kappa\,\mathcal R_{\mathrm{SA}}(\theta),
\)
where $\kappa\!\ge\!0$ toggles the regularization strength. Setting $\kappa\!=\!0$ disables SA regularization, yielding pruning-only training. 

\noindent\textbf{Perturbation sets.}
For state attacks we use $\ell_\infty$ balls 
$B(s)=\{\hat s:\,\|\hat s-s\|_\infty\le\varepsilon\}$ 
in normalized state space, matching the SA-MDP formulation 
and robust PPO practice, with environment--specific budgets 
$\varepsilon=0.075$ (\texttt{hopper}), 
$0.05$ (\texttt{walker2d}), 
$0.15$ (\texttt{ant}), and 
$0.15$ (\texttt{halfcheetah}).

\subsection{Adversarial Attacks}
\label{sec:attacks}
We evaluate robustness against four standard state-adversarial attacks, summarized in Table~\ref{tab:attacks}. Each attack is applied at every control step during rollouts.  

\begin{table}[t]
\centering
\caption{Summary of adversarial attacks used during training and evaluation.}
\label{tab:attacks}
\footnotesize
\begin{tabular}{p{3cm}p{10cm}}
\toprule
\textbf{Attack} & \textbf{Description} \\
\midrule
Random & Samples $\hat s$ uniformly from the perturbation set $B(s)$.\\
Value-guided & Perturbs states to minimize $V^\pi(s)$ via gradient descent on the critic.\\
MAD & Maximizes $D_{\mathrm{KL}}(\pi_\theta(\cdot\!\mid\! s)\,\|\,\pi_\theta(\cdot\!\mid\!\hat s))$ with projected gradient steps.\\
Robust Sarsa (RS) & Uses a robust TD update of $Q^\pi$ to find perturbations $\hat s$ that minimize $Q^\pi(s,\pi(\hat s))$.\\
\bottomrule
\end{tabular}
\end{table}

\subsection{Pruning Strategies}
\label{sec:pruning}

We compare five pruning strategies applied to both actor and critic networks:
Random (uniform weight removal under ERK allocation),
Magnitude (removing the smallest weights),
Magnitude--STE (magnitude pruning with straight-through estimator updates),
Saliency (based on first-order Taylor scores),
and a dense No-Pruning baseline.
All pruning methods (except the baseline) follow a cubic sparsity schedule after a $25\%$ burn-in.

\noindent\textbf{Micro-pruning.}
We call a schedule that increases sparsity through many small, frequent mask updates
\emph{micro-pruning}, as opposed to applying a single large pruning step.
The global target sparsity still follows a cubic schedule, but the mask is adjusted incrementally
so that the model is pruned in fine-grained steps rather than all at once.

\noindent\textbf{Why gradual steps help.}
The three–term bound in Sec.~\ref{sec:theory} (Theorem~\ref{thm:three-term}) shows that pruning
introduces a performance loss proportional to
$\mathcal L_{\mathrm{par}}(\theta,\theta')\,\|\Delta\theta\|$,
where $\mathcal L_{\mathrm{par}}$ is a path–averaged sensitivity measuring how much Jacobians vary
along the pruning trajectory.
Large pruning steps can push parameters through regions where sensitivities change sharply,
making $\mathcal L_{\mathrm{par}}$ large.
By contrast, small incremental steps keep consecutive parameters close,
so the Jacobian varies smoothly and the integrand in $\mathcal L_{\mathrm{par}}$ stays stable.
Thus micro-pruning tends to accumulate performance cost more gently, while still benefiting
from the monotone decrease in the Lipschitz bound $\tilde L_\theta$ that controls robustness.

\noindent\textbf{Sweet-spot definition.}
To quantify the joint effect of pruning on standard performance and robustness, 
we define the \emph{sweet spot} of each method as the pruning level at which 
the average of normalized clean and normalized robust performance is maximized. 
This captures the pruning regime where robustness improvements are realized 
without disproportionate loss in clean-task return, and is reported consistently 
across environments and methods.

%% file: sections/6_results.tex
\subsection{Empirical Analysis}
\label{sec:results}

All reported results in this section are \emph{normalized against the unpruned SA-trained network}, 
which serves as our dense baseline and are averaged over 5 seeds. This ensures pruning is always evaluated relative to the strongest 
non-sparse policy rather than a weaker vanilla PPO baseline.

Our theory establishes that pruning alone monotonically improves certified robustness (Theorem~\ref{thm:prune-certified-robust}). 
During training, however, gradient updates can enlarge weight norms and hence $\tilde L_\theta$, so robustness is not globally monotone (cf. Training Remark, Sec.~\ref{sec:theory}). 
Nevertheless, each pruning step strictly decreases $\tilde L_\theta$ relative to the current parameters, acting as a monotone regularizer that counteracts the natural growth of weight norms during training. 
From this perspective, one should not expect perfectly monotone empirical curves, but rather robustness that tends to increase steadily with pruning, punctuated by fluctuations from training noise. 
We now test this prediction across \texttt{hopper}, \texttt{halfcheetah}, and \texttt{walker2d}, gradually building a picture of how pruning reshapes the performance--robustness landscape.

\noindent\textbf{Clean vs. robust frontiers.}
We begin by examining the overall trade-off between clean-task performance and robustness. 
Figure~\ref{fig:frontiers_combined} shows these frontiers for \texttt{hopper} and \texttt{halfcheetah}. 
In \texttt{hopper}, robustness climbs until around 40\% pruning before clean-task degradation takes over. 
\texttt{halfcheetah} is strikingly more tolerant, maintaining clean-task performance up to $\sim$70\% pruning. 
These patterns reflect the three-term decomposition in Theorem~\ref{thm:three-term}: pruning reduces the Lipschitz gap term, but excessive sparsity eventually drives large parameter displacements that erode performance. 

By contrast, \texttt{walker2d} is far less forgiving: robustness initially rises but collapses past 50\%. 
These differences align with environment dynamics: \texttt{halfcheetah}’s smoother transitions allow redundancy, while \texttt{walker2d}’s instability amplifies sensitivity. 
Appendix~\ref{app:figures} provides the full set of curves, confirming the reproducibility of these trends.

\begin{figure}[t]
    \centering
    \begin{subfigure}[t]{\linewidth}
        \centering
        \includegraphics[width=\linewidth]{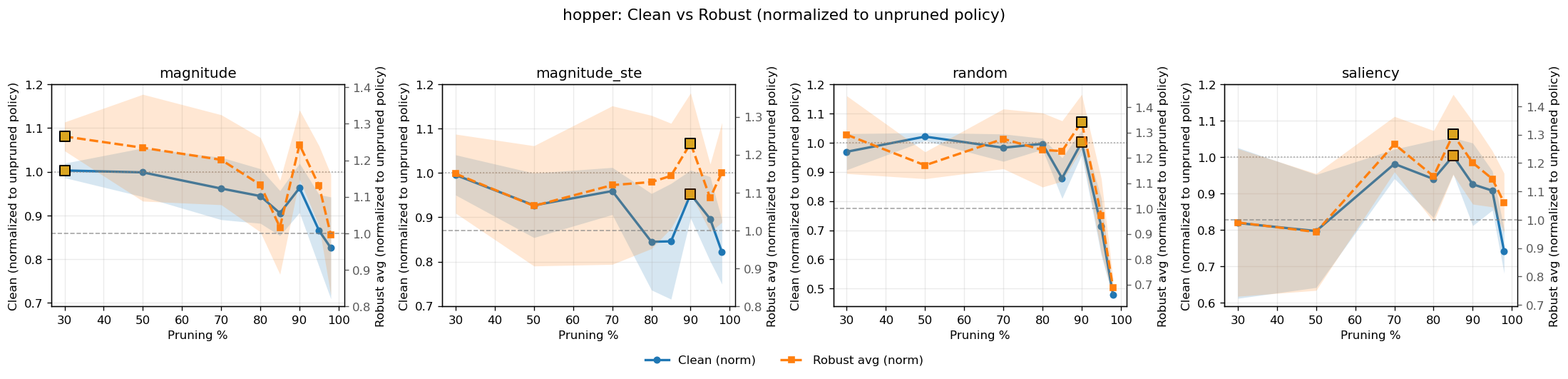}
    \end{subfigure}
    \vspace{0.5em}
    \begin{subfigure}[t]{\linewidth}
        \centering
        \includegraphics[width=\linewidth]{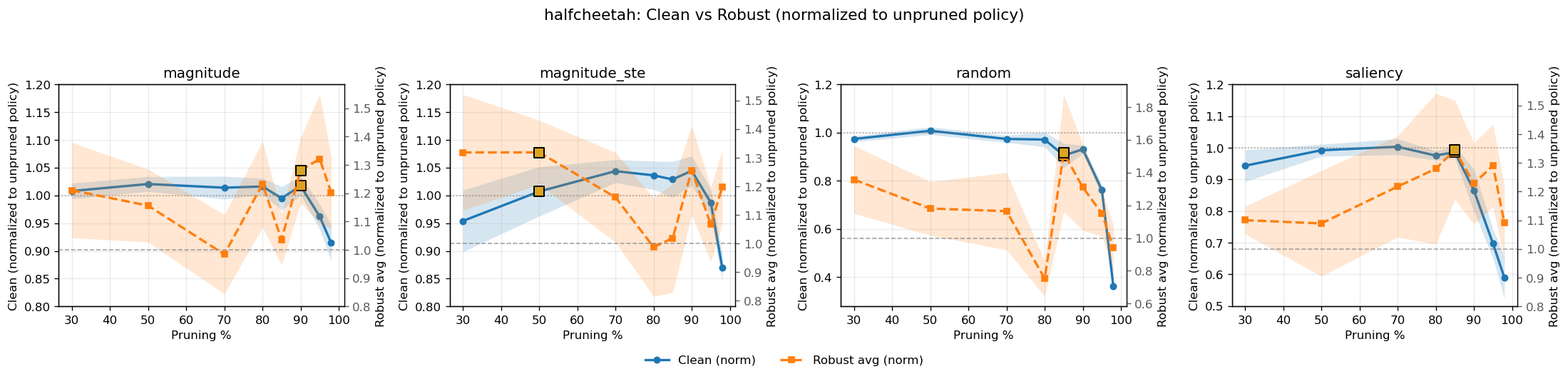}
    \end{subfigure}
    \caption{\textbf{Clean vs.\ robust frontiers under pruning.} 
    (Top) \texttt{hopper}: normalized clean and robust returns as pruning increases, across strategies. 
    (Bottom) \texttt{halfcheetah}: analogous trends with higher pruning tolerance. 
    All curves are normalized to the unpruned SA-trained policy; \emph{shaded regions denote $\pm$ one standard error across seeds}.}
    \label{fig:frontiers_combined}
\end{figure}

The other environments follow the same pattern but with different tolerances. 
\texttt{halfcheetah} is strikingly robust, maintaining clean-task performance up to $\sim$70\% pruning, whereas \texttt{walker2d} is far less forgiving: robustness initially rises but collapses past 50\%. 
These differences align with environment dynamics: \texttt{halfcheetah}’s smoother transitions allow redundancy, while \texttt{walker2d}’s instability amplifies sensitivity. 
Appendix~\ref{app:figures} provides the full set of curves, including seed variability, which confirm the reproducibility of these sweet spots.

The pruning method also matters. 
Magnitude pruning consistently yields the most stable frontiers, as expected from Theorem~\ref{thm:three-term} since it directly controls $\|\Delta \theta\|$. 
Saliency pruning looks competitive in \texttt{hopper} but breaks down in more complex environments, where instantaneous gradient saliency is a poor proxy for long-horizon contributions. 
Magnitude--STE introduces noise by pruning sensitive layers too aggressively, and random pruning is unsurprisingly the least reliable: it occasionally boosts robustness but often destroys clean-task returns.

\noindent\textbf{Attack-specific robustness.}
To understand robustness more finely, we next examine performance under different adversaries. 
Figure~\ref{fig:hopper_attacks} shows \texttt{hopper} curves under four state-adversarial attacks (Random, Value-guided, MAD, RS). 
Here a clearer picture emerges: pruning offers the strongest gains against broad-spectrum adversaries (MAD and RS), boosting robust returns by several hundred reward points in the 40--60\% sparsity range. 
By contrast, targeted Value-guided attacks are less affected, producing noisier or weaker gains. 
This contrast reflects the gap between global and local robustness: pruning reduces the global Lipschitz constant, but does not eliminate vulnerabilities to specific input patterns. 
In other words, pruning hardens the policy against generic perturbations, but some adversary-specific weaknesses remain. 
\texttt{halfcheetah} and \texttt{walker2d} exhibit the same qualitative trends (Appendix~\ref{app:figures}), though the precise sweet spot again depends on environment dynamics.

\begin{figure}[t]
    \centering
    \includegraphics[width=\linewidth]{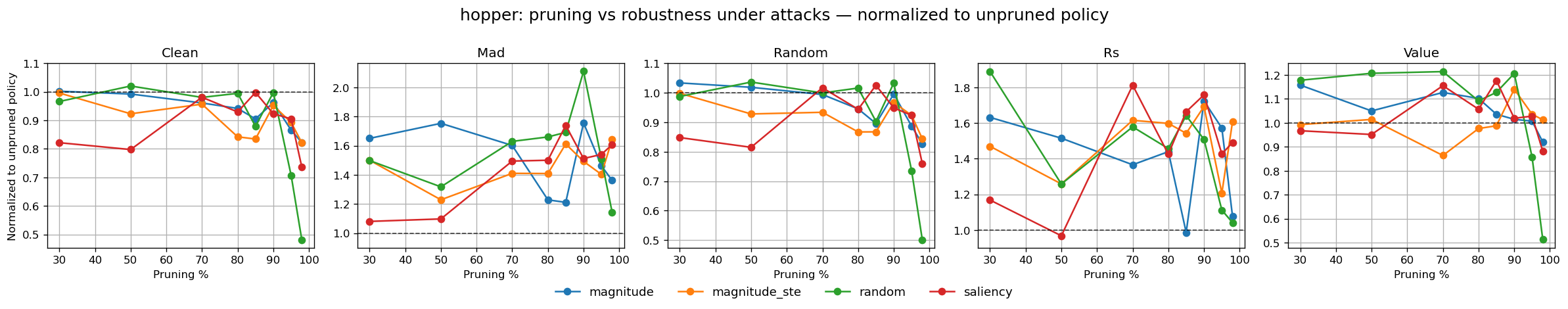}
    \caption{\textbf{\texttt{hopper} under attack.} 
    Robustness gains are strongest against MAD and RS adversaries, consistent with pruning’s global Lipschitz guarantee. 
    Improvements are smaller and less consistent against targeted Value-guided attacks. 
    Appendix~\ref{app:figures} shows analogous plots for \texttt{halfcheetah} and \texttt{walker2d}.}
    \label{fig:hopper_attacks}
\end{figure}

\noindent\textbf{Effect of adversarial training.}
A natural question is whether pruning simply mimics the effect of adversarial (SA) training. 
Figure~\ref{fig:saonoff} compares \texttt{hopper} returns with and without SA regularization. 
We find that pruning consistently improves robustness in both settings, confirming that it acts as an independent structural bias. 
The incremental effect of SA training under pruning is modest and attack--dependent 
(e.g., clearer under RS and Random, negligible under Clean and Value). 
This suggests that pruning and SA are not interchangeable, but their combination does not always yield additive gains.
\begin{figure}[t]
    \centering
    \includegraphics[width=\linewidth]{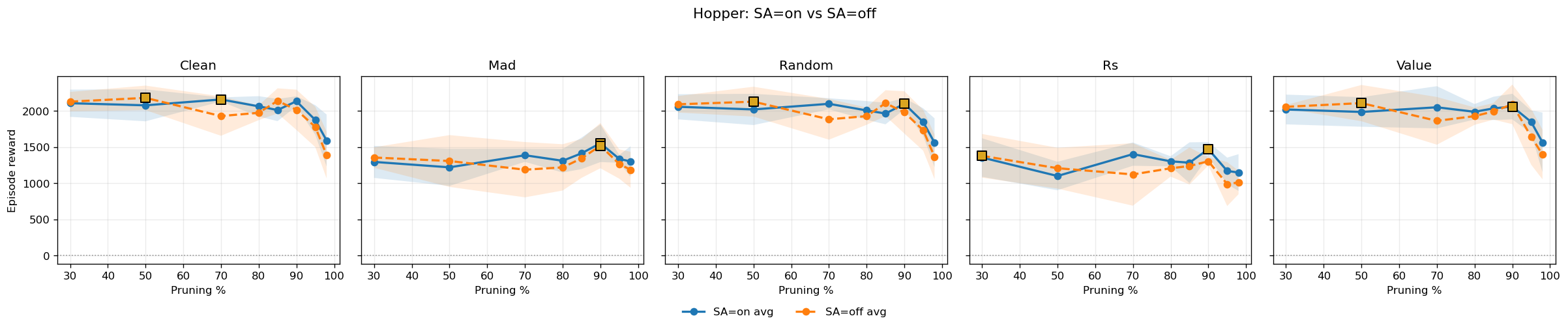}
    \caption{\textbf{Pruning vs. adversarial training (\texttt{hopper}).} 
    Pruning yields robustness gains in both regimes. 
    SA-regularization sometimes provides additional improvements (notably under RS and Random), but the effect is uneven across attacks.}
    \label{fig:saonoff}
\end{figure}

\noindent\textbf{Micropruning ablation.}
When pruning is interleaved with training, we observe that robustness gains 
and clean-task performance often evolve in parallel (Fig.~\ref{fig:micropruning})
Micro-pruning schedules, which update the pruning mask in small increments, 
allow the network to adjust gradually: each incremental step reduces the 
Lipschitz bound controlling robustness, while ongoing weight updates help 
offset the associated parameter change. 
As a result, performance curves remain smoother and robustness improvements 
are more stable, with sweet spots emerging at intermediate sparsity levels.
Figure~\ref{fig:micropruning} illustrates this effect across pruning intervals. 
Applying mask updates every 10--20 steps yields the most stable curves, while pruning every single step introduces more variability due to interaction with gradient noise. 
The same qualitative pattern holds across environments, with \texttt{hopper} benefitting most clearly from 20-step pruning, while \texttt{halfcheetah} and \texttt{walker2d} stabilize at 10--15 steps.
\begin{figure}[t]
    \centering
    \includegraphics[width=\linewidth]{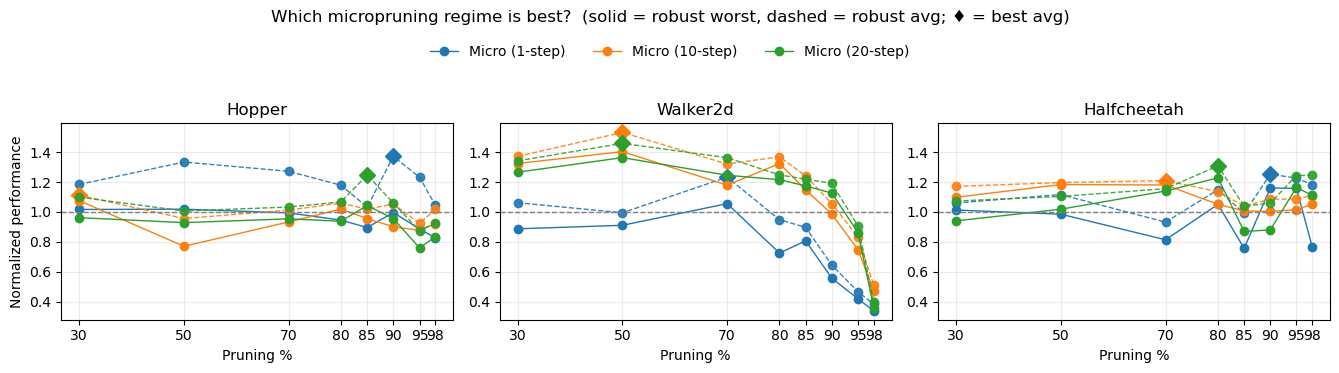}
    \caption{\textbf{Micropruning ablation.} 
    Updating masks in small, periodic increments (10--20 steps) leads to smoother curves 
    and more reliable sweet spots than pruning every step.}
    \label{fig:micropruning}
\end{figure}
\\ \\ \noindent\textbf{Sweet spot quantification.}
Finally, Appendix Table~\ref{tab:sweetspots} quantifies the sweet-spot sparsities across environments and pruning methods. 
\texttt{hopper} peaks around 30--50\%, \texttt{halfcheetah} around 50--70\%, and \texttt{walker2d} around 30--50\%. 
Notably, magnitude pruning consistently finds these ranges, while random pruning is far more variable. 
Across tasks, pruning improves normalized robust performance by $1.1\times$--$1.6\times$ relative to baseline, showing that the benefits are substantial and reproducible. 
Appendix Table~\ref{tab:avg_worst_seed_per_method_with_percent} further reports per-seed worst-case values, confirming that these gains are not driven by lucky seeds but persist across training runs.

%% file: sections/7_conclusion.tex
\section{Conclusion and Future Work}
\label{sec:conclusion}
We studied \emph{element--wise} pruning in reinforcement learning and showed, both theoretically and empirically, that it acts as a monotone regularizer; each pruning step reduces a Lipschitz surrogate of robustness, while performance loss is captured by a three--term regret bound. 
Across continuous--control benchmarks, we consistently observe reproducible ``sweet spots'' where robustness gains outweigh clean-task degradation, under both standard and adversarial training. 
To the best of our knowledge, this is the first work to certify that pruning in RL can never reduce robustness, positioning it not only as a compression tool but as a structural intervention shaping the performance--robustness trade-off.

\noindent\textbf{Limitations and Future Work.}
Our study focuses on continuous–control benchmarks with MLP policies and element–wise pruning. While this setting offers a clean testbed, it leaves open important questions regarding generality. Extending the theory and experiments to pixel–based environments and richer architectures (e.g., CNNs, RNNs, or transformers) is a natural next step. Similarly, investigating structured pruning methods—such as neuron, channel, or layer pruning—could provide more practical compression gains and richer robustness–performance trade-offs. Another promising direction is to integrate pruning more tightly with the training process. For instance, jointly optimizing pruning with adversarial or robust training may yield complementary benefits, while data–dependent sensitivity estimates could enable sharper and more adaptive pruning schedules. Beyond the static adversarial models considered here, it is also important to evaluate pruning under richer and more realistic threat models, including adaptive, temporally correlated, or non–stationary attacks. Finally, while we have established pruning as a robustness–preserving intervention, its broader implications for policy generalization, exploration, and sample efficiency remain underexplored. Addressing these questions would help clarify when and how pruning can serve not only as a compression tool but as a principled means of shaping learning dynamics in reinforcement learning.

\section*{Acknowledgements}

James Pedley is funded through the Willowgrove Studentship. Francesco Quinzan is funded through the Oxford Martin School. The authors also acknowledge support from His Majesty’s Government in the development of this research.

%% file: sections/x_appendices.tex
\appendix 

\section{Algorithm}
\label{app:algo}
\begin{algorithm}
\caption{PPO with Pruning and Optional SA Regularization}
\label{alg:prune-sa-ppo}
\begin{algorithmic}[1]
\Require 
  Initial policy parameters $\theta$; Pruning masks $\{M_\ell\}$ for each layer (initialized as all ones, i.e.\ no pruning applied); 
  Adversarial budget $B(s)$ (state perturbation set); 
  Regularization weight $\kappa \ge 0$; 
  Total update horizon $T$; 
  Burn-in fraction $\beta \in [0,1]$; 
  Choice of pruning rule. 
\State Initialize update counter $t=0$
\For{each update iteration}
  \State Collect trajectories using adversarial states $\hat s \in B(s)$
  \State Compute advantages $\hat A$ and returns $\hat R$
  \For{each minibatch $\mathcal B$}
    \State $t \gets t+1$
    \State Compute PPO loss with SA regularisation 
      $\mathcal L \gets \mathcal L_{\mathrm{PPO}} + \kappa\,\mathcal R_{\mathrm{SA}}(\theta;\mathcal B)$ 
    \State Update network parameters with gradient descent
    \If{$t/T > \beta$ (network burn-in)}
      \State Update pruning masks $\{M_\ell\}$ according to chosen rule
      \State Apply masks to network parameters \Comment{pruning step}
    \EndIf
  \EndFor
\EndFor
\end{algorithmic}
\end{algorithm}

\makeatletter
\renewenvironment{proof}[1][Proof]{\par
  \pushQED{\qed}\normalfont
  \trivlist\item[\hskip\labelsep\bfseries #1.]\ }{\popQED\endtrivlist}
\makeatother

\section{Proofs}
\label{app:proofs}
\subsection{Proof of Theorem~\ref{thm:prune-certified-robust}}

\begin{proof}
For any policy $\pi$ and its optimal adversary $\nu^*(\pi)$ in the state-adversarial MDP (SA-MDP), it holds from \citet{DBLP:conf/nips/0001CX0LBH20} that
\begin{equation}\label{eq:sa-tv}
\max_{s}\big\{V^\pi(s)-\tilde V^{\pi\circ\nu^*(\pi)}(s)\big\}
\;\le\;\alpha \max_{s}\max_{\hat s\in B(s)}D_{\mathrm{TV}}\!\big(\pi(\cdot|s),\pi(\cdot|\hat s)\big),
\end{equation}
where $B(s)=\{\,\hat s:\|\hat s-s\|_2\le\varepsilon\,\}$ is the $\ell_2$ perturbation ball, and $D_{\mathrm{TV}}$ denotes total variation distance.  

\paragraph{Network Lipschitz bound.}
Let the policy network be an $L$-layer feedforward model with parameters $\theta=\{W_1,\dots,W_L\}$, biases $\{b_\ell\}$, and $L_{\sigma_\ell}$–Lipschitz activations $\sigma_\ell$. Explicitly, the network map $f_\theta:\mathcal S\to\mathbb R^d$ is the function composition
\[
f_\theta(s) \;=\; W_L \,\sigma_{L-1}\!\Big(W_{L-1}\,\sigma_{L-2}\big(\cdots \sigma_1(W_1 s + b_1)+b_{L-1}\big)\Big) + b_L,
\]
where $\sigma_\ell$ is applied elementwise. Biases do not affect Lipschitz constants.

The Lipschitz constant of a linear map $x\mapsto W_\ell x$ is its operator (spectral) norm $\|W_\ell\|_2$, since
\[
\|W_\ell x - W_\ell y\|_2 = \|W_\ell(x-y)\|_2 \le \|W_\ell\|_2 \|x-y\|_2.
\]
Thus the Lipschitz constant of $f_\theta$ is bounded by
\[
\mathrm{Lip}(f_\theta) \;\le\;\Bigg(\prod_{\ell=1}^{L-1}L_{\sigma_\ell}\Bigg)\,\prod_{\ell=1}^L \|W_\ell\|_2.
\]

Since computing or constraining $\|W_\ell\|_2$ can be difficult, we introduce monotone surrogates. For any matrix $A$,
\[
\|A\|_2 \;\le\; \|A\|_F, 
\qquad
\|A\|_2 \;\le\; \sqrt{\|A\|_1 \,\|A\|_\infty}.
\]
These upper bounds are monotone in the entries of $A$, hence suitable for analyzing pruning. Therefore,
\[
\mathrm{Lip}(f_\theta)
\;\le\;\Bigg(\prod_{\ell=1}^{L-1}L_{\sigma_\ell}\Bigg)\,
       \prod_{\ell=1}^L \min\!\Big\{\,\|W_\ell\|_F,\;\sqrt{\|W_\ell\|_1\|W_\ell\|_\infty}\,\Big\}
=:\tilde L_\theta.
\]
Thus for any $s,\hat s\in B(s)$,
\[
\|f_\theta(\hat s)-f_\theta(s)\|_2 \;\le\;\tilde L_\theta\,\varepsilon.
\]

\paragraph{Gaussian policies.}
For $\pi_\theta(a|s)=\mathcal N(\mu_\theta(s),\Sigma)$ with fixed $\Sigma\succ0$, the closed-form total variation distance between Gaussians with equal covariance gives
\[
D_{\mathrm{TV}}\!\big(\pi_\theta(\cdot|s),\pi_\theta(\cdot|\hat s)\big)
\;\le\;\frac{\|\mu_\theta(\hat s)-\mu_\theta(s)\|_2}{\sqrt{2\pi\,\lambda_{\min}(\Sigma)}}
\;\le\;\tfrac{1}{\sqrt{2\pi\,\lambda_{\min}(\Sigma)}}\,\tilde L_\theta\,\varepsilon.
\]

\paragraph{Categorical policies.}
For $\pi_\theta(\cdot|s)=\mathrm{softmax}(z_\theta(s))$, the log-partition function is $1/4$-smooth, which yields the standard bound
\[
D_{\mathrm{TV}}(\pi_\theta(\cdot|s),\pi_\theta(\cdot|\hat s))
\;\le\;\tfrac14\,\|z_\theta(s)-z_\theta(\hat s)\|_2
\;\le\;\tfrac14\,\tilde L_\theta\,\varepsilon.
\]

\paragraph{Effect of pruning.}
Let $\theta'$ be obtained by elementwise pruning, $W_\ell'=\mathcal M_\ell\odot W_\ell$ with binary masks $\mathcal M_\ell$. Each surrogate norm is monotone under pruning:
\[
\|W_\ell'\|_F \le \|W_\ell\|_F,
\qquad
\sqrt{\|W_\ell'\|_1\|W_\ell'\|_\infty}\;\le\;\sqrt{\|W_\ell\|_1\|W_\ell\|_\infty}.
\]
Hence $\tilde L_{\theta'}\le \tilde L_\theta$. Activation Lipschitz constants are unchanged, so the same bounds apply with $\tilde L_{\theta'}$, which is no larger.  

Therefore, for Gaussian ($c=\tfrac{1}{\sqrt{2\pi\,\lambda_{\min}(\Sigma)}}$) and categorical ($c=\tfrac14$) policies,
\[
\max_s \{V^{\pi_{\theta'}}(s)-\tilde V^{\pi_{\theta'}\circ\nu^*(\pi_{\theta'})}(s)\}
\;\le\;\alpha\,c\,\tilde L_{\theta'}\,\varepsilon
\;\le\;\alpha\,c\,\tilde L_\theta\,\varepsilon.
\]
Thus pruning cannot worsen the certified robustness bound.
\end{proof}

\noindent\textbf{Tightness of the bound.}
The surrogate Lipschitz bound \(\tilde L_\theta\) from Theorem~\ref{thm:prune-certified-robust} is guaranteed to be
monotone under entrywise pruning, but it can be loose compared to the
true sensitivity of the policy. A sharper, distribution--dependent refinement is
given in Lemma~\ref{lem:local-bound}:
\[
TV(\pi_\theta(\cdot\mid s),\,\pi_\theta(\cdot\mid s+\epsilon))
\;\le\; c\Big(\|J_{g_\theta}(s)\|_{\mathrm{op}}\,\epsilon+\tfrac12\beta\,\epsilon^2\Big),
\]
where \(\beta\) is a curvature constant capturing the local variation of the Jacobian,
e.g.\ an upper bound on the Lipschitz constant of \(J_{g_\theta}(s)\). This local bound is (trivially)
always no larger than the global bound (exact for ReLU networks, up to an
\(O(\epsilon^2)\) term otherwise), and often much tighter since typical Jacobians have small operator norm. However,
unlike \(\tilde L_\theta\), it is not guaranteed to decrease monotonically under
pruning. Thus, the global bound provides provable monotone improvement,
while the local refinement better reflects the true robustness landscape but may vary non-monotonically.

\begin{lemma}[Local robustness bound]\label{lem:local-bound}
Let $\pi_\theta$ be either a Gaussian policy 
$\pi_\theta(a\mid s)=\mathcal{N}(\mu_\theta(s),\Sigma)$ 
with fixed $\Sigma \succ 0$, 
or a categorical policy 
$\pi_\theta(\cdot\mid s) = \mathrm{softmax}(z_\theta(s))$. 
Suppose the network outputs $g_\theta(s)$ 
(mean $\mu_\theta(s)$ or logits $z_\theta(s)$) are $\beta$-smooth, i.e.,
$\|J_{g_\theta}(x)-J_{g_\theta}(y)\|_{\mathrm{op}}\le \beta\|x-y\|_2$ for all $x,y$.
Then for any perturbation $\epsilon$ and state $s$,
\[
TV\!\big(\pi_\theta(\cdot\mid s),\,\pi_\theta(\cdot\mid s+\epsilon)\big)
\;\le\; c\Big(\|J_{g_\theta}(s)\|_{\mathrm{op}}\,\|\epsilon\|_2
+\tfrac12\,\beta\,\|\epsilon\|_2^2\Big),
\]
where $J_{g_\theta}(s)$ is the Jacobian of $g_\theta$ at $s$ (operator norm induced by $\ell_2$),
and $c=\tfrac{1}{\sqrt{2\pi\,\lambda_{\min}(\Sigma)}}$ for Gaussians and $c=\tfrac14$ for categoricals.
\end{lemma}

\begin{proof}
By first-order Taylor's theorem with integral remainder and the $\beta$-smoothness of $g_\theta$,
\[
g_\theta(s+\epsilon)
= g_\theta(s) + J_{g_\theta}(s)\,\epsilon 
+ \underbrace{\int_0^1\!\big(J_{g_\theta}(s+t\epsilon)-J_{g_\theta}(s)\big)\,\epsilon\,\mathrm{d}t}_{r(\epsilon)},
\]
and hence
\[
\|r(\epsilon)\|_2 
\;\le\;\int_0^1 \|J_{g_\theta}(s+t\epsilon)-J_{g_\theta}(s)\|_{\mathrm{op}}\;\|\epsilon\|_2\,\mathrm{d}t
\;\le\;\int_0^1 \beta\,t\,\|\epsilon\|_2^2\,\mathrm{d}t
\;=\;\tfrac12\,\beta\,\|\epsilon\|_2^2.
\]
Therefore,
\[
\|g_\theta(s+\epsilon)-g_\theta(s)\|_2
\;\le\;\|J_{g_\theta}(s)\|_{\mathrm{op}}\,\|\epsilon\|_2+\tfrac12\,\beta\,\|\epsilon\|_2^2.
\]

For Gaussians with identical covariance $\Sigma\succ0$, the closed-form total variation bound gives
\(
TV\!\big(\pi_\theta(\cdot\mid s),\pi_\theta(\cdot\mid s+\epsilon)\big)
\le \|\,\mu_\theta(s+\epsilon)-\mu_\theta(s)\|_2/\sqrt{2\pi\,\lambda_{\min}(\Sigma)}.
\)
For categoricals, the softmax log-partition is $1/4$-smooth, yielding
\(
TV\!\big(\pi_\theta(\cdot\mid s),\pi_\theta(\cdot\mid s+\epsilon)\big)
\le \tfrac14 \|\,z_\theta(s+\epsilon)-z_\theta(s)\|_2.
\)
Applying these with $g_\theta$ as the mean or logits respectively establishes the claim.
\end{proof}

\subsection{Lemma~\ref{thm:unified-regret}}
\begin{lemma}[Expected robustness gap]
\label{thm:expected-tv}
For any start-state distribution $\mu$,
\[
\mathbb{E}_{s_0\sim \mu}[V_{\pi_\theta}(s_0)-\tilde V_{\pi_\theta\circ \nu^*}(s_0)]
\;\le\;
\alpha\,\mathbb{E}_{s\sim d_\mu^{\pi_\theta}}\!\big[TV(\pi_\theta(\cdot\mid s),\pi_\theta(\cdot\mid \nu^*(s)))\big]
\;\le\; \mathcal B(\theta).
\]
\end{lemma}

\begin{proof}
The first inequality is obtained from \citet{DBLP:conf/nips/0001CX0LBH20} by taking the expectation over the difference in values instead of \(max_s\). The constant $\alpha$ (defined in the main paper) collects the reward bound and $\gamma$.  

For the second inequality, note that by definition
\[
TV(\pi_\theta(\cdot\mid s),\pi_\theta(\cdot\mid \nu^*(s))) \;\le\; TV_{\max}(s;\theta).
\]
Taking expectations over $s\sim d_\mu^{\pi_\theta}$ yields
\[
\mathbb{E}_{s}\Big[TV(\pi_\theta(\cdot\mid s),\pi_\theta(\cdot\mid \nu^*(s)))\Big]
\;\le\; \mathbb{E}_{s}\big[TV_{\max}(s;\theta)\big]
= F(\theta).
\]
Multiplying by $\alpha$ gives the claimed bound $\mathcal{B}(\theta)=\alpha F(\theta)$.
\end{proof}

\subsection{Lemma~\ref{thm:value-drop}}

\begin{lemma}[Clean/attacked value drop under pruning]
\label{thm:value-drop}
Let $\pi_\theta$ be Gaussian with fixed covariance $\Sigma\succ0$ or categorical softmax with logits $z_\theta(s)$ from a Lipschitz network.
Assume that for each $s$, the map $\phi\mapsto g_\phi(s)$ is differentiable almost everywhere.
For any pruned parameters $\theta'=\theta-\Delta\theta$, define
\[
\hat L^{\mathrm{par}}_\phi
:=\Big(\mathbb{E}_{s\sim d_\mu^{\pi_\theta}}\|J_\phi g_\phi(s)\|_{\mathrm{op}}^2\Big)^{1/2},
\qquad
\mathcal L_{\mathrm{par}}(\theta,\theta') := \int_0^1 \hat L^{\mathrm{par}}_{\theta'+t(\theta-\theta')}\,\mathrm{d}t.
\]
With $c=\tfrac{1}{\sqrt{2\pi\,\lambda_{\min}(\Sigma)}}$, \(g_{\phi} = \mu_{\phi}\) for Gaussian and $c=\tfrac14$, \(g_{\phi} = z_{\phi}\) for categorical, we have
\[
J(\pi_\theta)-J(\pi_{\theta'})
\;\le\;\alpha\,c\,\mathcal L_{\mathrm{par}}(\theta,\theta')\,\|\Delta\theta\|,
\qquad
\tilde J(\pi_\theta\!\circ\!\nu^*)-\tilde J(\pi_{\theta'}\!\circ\!\nu^*)
\;\le\;\alpha\,c\,\mathcal L_{\mathrm{par}}(\theta,\theta')\,\|\Delta\theta\|.
\]
(Here $\|\cdot\|$ on parameters is Euclidean, and $\|\cdot\|_{\mathrm{op}}$ is the operator norm induced by $\ell_2$.) 
\end{lemma}

\begin{proof}
By the SA–MDP value–difference bound (Lemma~\ref{thm:expected-tv} in TV form) applied to two policies at the same state,
\[
J(\pi_\theta)-J(\pi_{\theta'}) \;\le\; \alpha\,\mathbb{E}_{s\sim d_\mu^{\pi_\theta}}
\Big[\,TV\!\big(\pi_\theta(\cdot|s),\pi_{\theta'}(\cdot|s)\big)\Big].
\]

\textbf{Gaussian:} For Gaussians with identical covariance $\Sigma\succ0$, the closed-form TV bound is
\[
TV\!\big(\pi_\theta(\cdot|s),\pi_{\theta'}(\cdot|s)\big)
\;\le\; \frac{\|\mu_\theta(s)-\mu_{\theta'}(s)\|_2}{\sqrt{2\pi\,\lambda_{\min}(\Sigma)}}.
\]

\textbf{Categorical:} For $\pi=\mathrm{softmax}(z)$, the log-partition is $1/4$-smooth, yielding
\[
TV(\pi_\theta(\cdot|s),\pi_{\theta'}(\cdot|s)) \;\le\; \tfrac14 \|z_\theta(s)-z_{\theta'}(s)\|_2.
\]

Thus, in both cases,
\[
TV\!\big(\pi_\theta(\cdot|s),\pi_{\theta'}(\cdot|s)\big)\;\le\; c\,\|g_\theta(s)-g_{\theta'}(s)\|_2.
\]

Now let $\phi(t):=\theta'+t(\theta-\theta')$, $t\in[0,1]$. Since $g_\phi(s)$ is (a.e.) differentiable in $\phi$, the fundamental theorem of calculus along $\phi(t)$ gives
\[
g_\theta(s)-g_{\theta'}(s)=\int_0^1 J_{\phi(t)}g_{\phi(t)}(s)\,(\theta-\theta')\,\mathrm{d}t.
\]
Taking norms and using the operator norm,
\[
\|g_\theta(s)-g_{\theta'}(s)\|_2\;\le\;\int_0^1 \|J_{\phi(t)}g_{\phi(t)}(s)\|_{\mathrm{op}}\,\mathrm{d}t\;\|\Delta\theta\|.
\]
By Tonelli/Fubini to exchange expectation and integral, and Cauchy–Schwarz in $s$,
\[
\mathbb{E}_{s}\|g_\theta(s)-g_{\theta'}(s)\|_2
\;\le\;\int_0^1 \Big(\mathbb{E}_{s}\|J_{\phi(t)}g_{\phi(t)}(s)\|_{\mathrm{op}}^2\Big)^{1/2} \mathrm{d}t\;\|\Delta\theta\|
\;=\;\mathcal L_{\mathrm{par}}(\theta,\theta')\,\|\Delta\theta\|.
\]
Combining with the TV inequality yields the claimed clean-value bound with factor $\alpha c$. The attacked-value bound is identical with $\tilde J$, since Lemma~\ref{thm:expected-tv} holds for robust values with the same TV control.
\end{proof}

\subsection{Proof of Theorem~\ref{thm:three-term}}
\begin{proof}
Decompose
\[
\mathrm{Reg}_{\mathrm{atk}}(\theta';\bar\pi)
= J(\bar\pi)-\tilde J(\pi_{\theta'}\!\circ\!\nu^*)
= \big[J(\bar\pi)-J(\pi_\theta)\big]
+ \big[J(\pi_\theta)-J(\pi_{\theta'})\big]
+ \big[J(\pi_{\theta'})-\tilde J(\pi_{\theta'}\!\circ\!\nu^*)\big].
\]
The first term is \(\mathrm{Reg}_{\mathrm{clean}}(\theta;\bar\pi)\). The second term is bounded by Lemma~2:
\(J(\pi_\theta)-J(\pi_{\theta'})\le \alpha c\,\mathcal L_{\mathrm{par}}(\theta,\theta')\,\|\Delta\theta\|\).
For the third term, apply Theorem~\ref{thm:prune-certified-robust} to \(\pi_{\theta'}\):
\(J(\pi_{\theta'})-\tilde J(\pi_{\theta'}\!\circ\!\nu^*)\le \alpha c\,\tilde L_{\theta'}\,\varepsilon\).
Summing the bounds yields the claim.
\end{proof}

\section{Additional results}
\label{app:figures}

This appendix provides the full set of figures and tables referenced in the main text.

\begin{figure}[H]
    \centering
    \includegraphics[width=\linewidth]{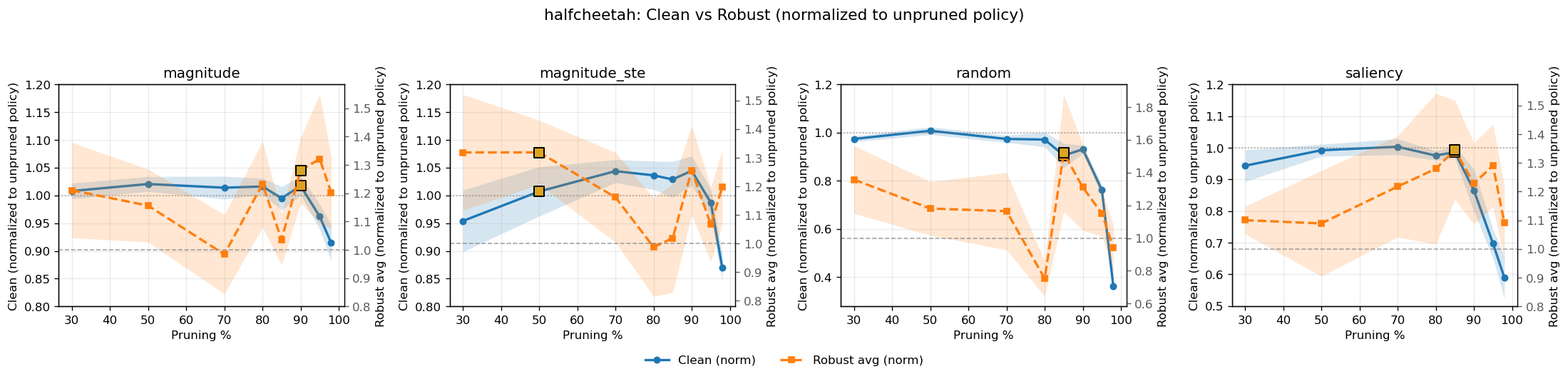}
    \caption{\textbf{\texttt{halfcheetah}: Clean vs.\ Robust frontier.} 
    Normalized clean and robust returns as pruning increases, across pruning strategies.}
    \label{fig:hc_t}
\end{figure}

\begin{figure}[H]
    \centering
    \includegraphics[width=\linewidth]{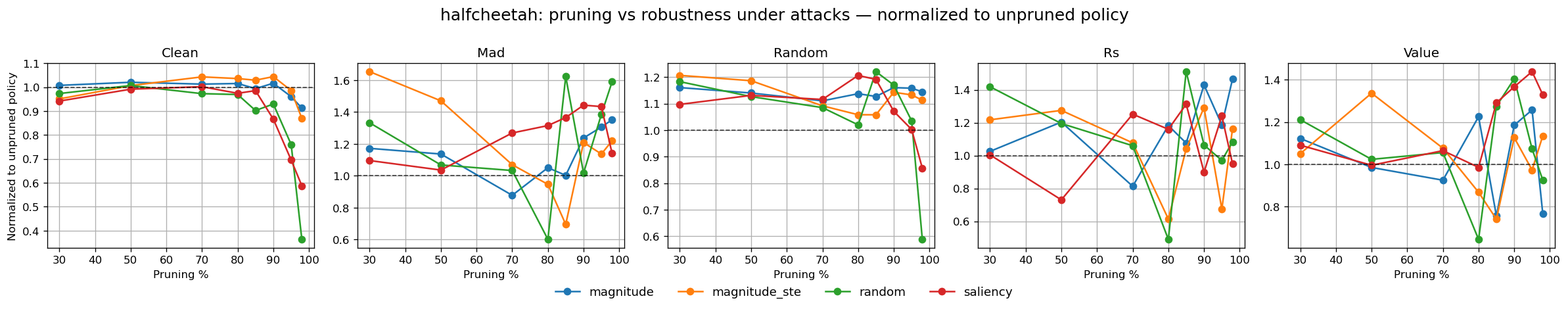}
    \caption{\textbf{\texttt{halfcheetah}: Robustness under different adversaries.} 
    Pruning vs.\ robustness curves across attack types (Clean, MAD, Random, RS, Value).}
    \label{fig:hc_attacks}
\end{figure}

\begin{figure}[H]
    \centering
    \includegraphics[width=\linewidth]{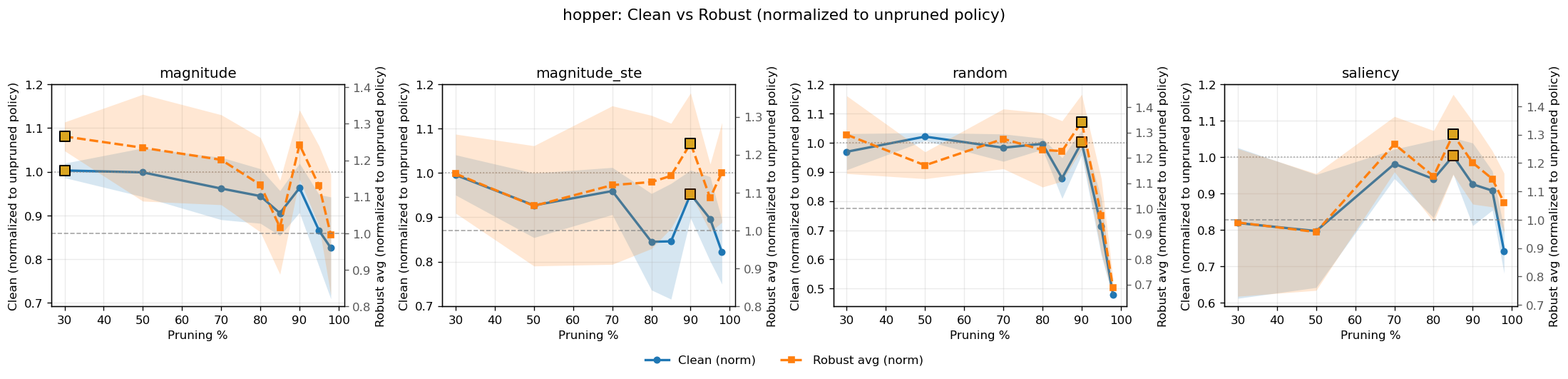}
    \caption{\textbf{\texttt{hopper}: Clean vs.\ Robust frontier.} 
    Normalized clean and robust returns as pruning increases, across pruning strategies.}
    \label{fig:hopper_t_app}
\end{figure}

\begin{figure}[H]
    \centering
    \includegraphics[width=\linewidth]{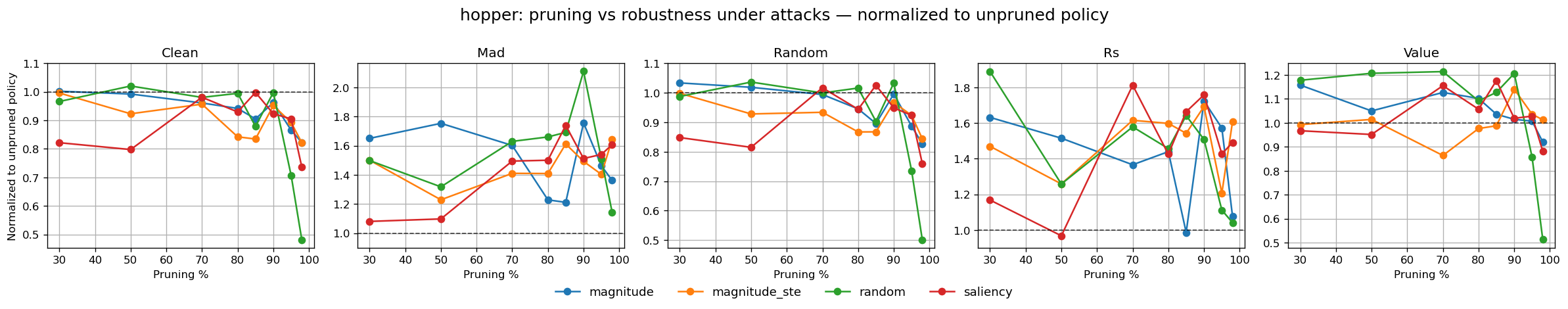}
    \caption{\textbf{\texttt{hopper}: Robustness under different adversaries.} 
    Pruning vs.\ robustness curves across attack types (Clean, MAD, Random, RS, Value).}
    \label{fig:hopper_attacks_app}
\end{figure}

\begin{figure}[H]
    \centering
    \includegraphics[width=\linewidth]{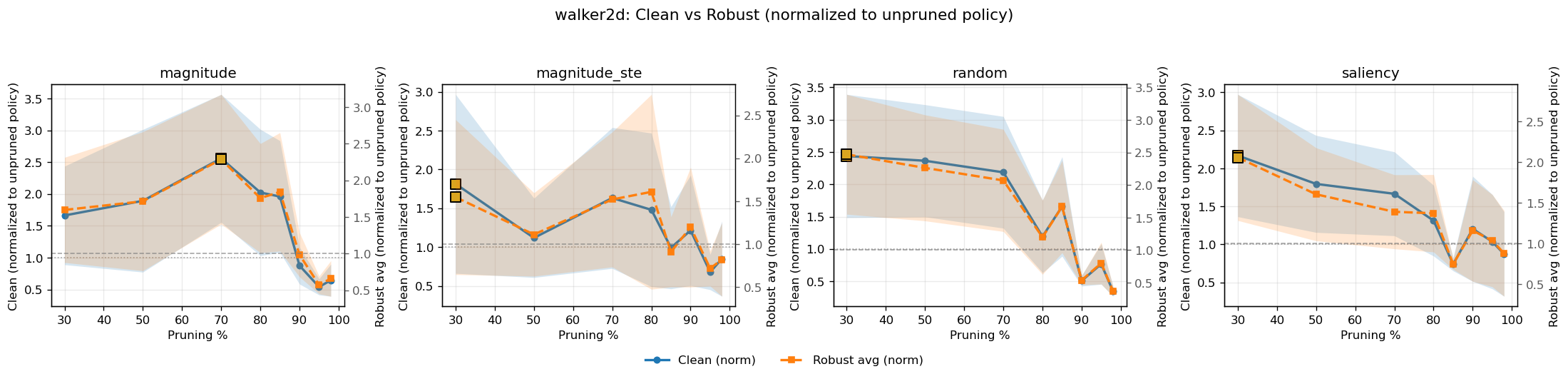}
    \caption{\textbf{\texttt{walker2d}: Clean vs.\ Robust frontier.} 
    Normalized clean and robust returns as pruning increases, across pruning strategies.}
    \label{fig:walker_t}
\end{figure}

\begin{figure}[H]
    \centering
    \includegraphics[width=\linewidth]{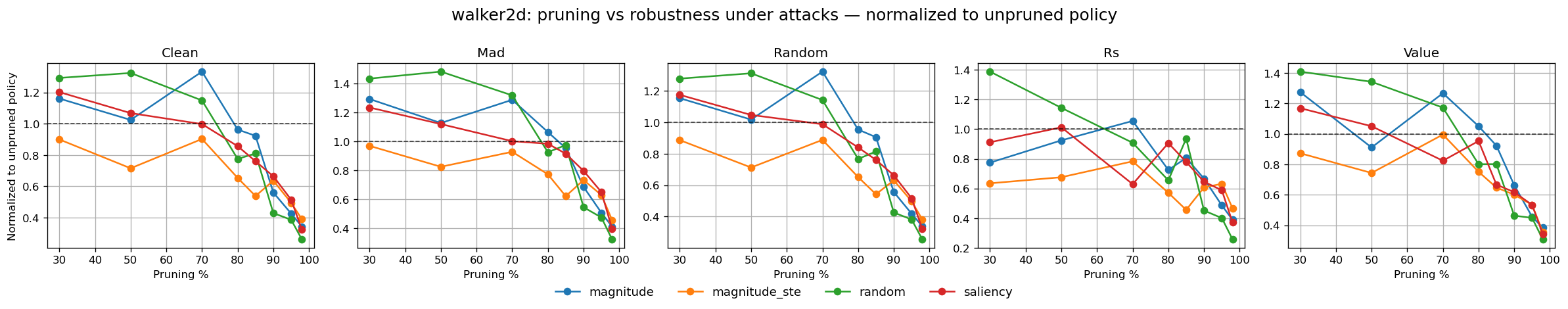}
    \caption{\textbf{\texttt{walker2d}: Robustness under different adversaries.} 
    Pruning vs.\ robustness curves across attack types (Clean, MAD, Random, RS, Value).}
    \label{fig:walker_attacks}
\end{figure}

\begin{table}[t]
    \centering
    \footnotesize
    \begin{tabular}{l l r r r r}
\toprule
Environment & Method & Sweet-Spot \%& Clean &  Robust (Worst) & Norm. Robust (Avg) \\
\midrule
 \texttt{hopper} &      magnitude &                30\% &        1.00 &                 1.63 &               1.26 \\
     &  magnitude\_ste &                90\% &        0.95 &                 1.53 &               1.22 \\
     &         random &                90\% &        1.00 &                 1.51 &               1.33 \\
     &       saliency &                85\% &        1.00 &                 1.66 &               1.28 \\
   \texttt{walker2d} &      magnitude &                70\% &        1.33 &                 1.09 &               1.24 \\
    &  magnitude\_ste &                70\% &        0.90 &                 0.81 &               0.90 \\
    &         random &                30\% &        1.29 &                 1.43 &               1.37 \\
    &       saliency &                30\% &        1.20 &                 0.94 &               1.13 \\
 \texttt{halfcheetah} &      magnitude &                90\% &        1.02 &                 1.24 &               1.23 \\
     &  magnitude\_ste &                50\% &        1.01 &                 1.47 &               1.28 \\
     &         random &                85\% &        0.90 &                 1.62 &               1.35 \\
     &       saliency &                85\% &        0.99 &                 1.37 &               1.26 \\
\bottomrule
\end{tabular}
    \caption{\textbf{Sweet-spot pruning levels.} 
    Across environments, pruning uncovers reproducible sparsity ranges (30--70\%) where robustness gains dominate without harming clean returns. Returns are normalized to the unpruned policy performance.}
    \label{tab:sweetspots}
\end{table}

\begin{table}[H]
\centering
\begin{tabular}{l l r r r}
\toprule
Environment & Method & Pruning \% & Avg. worst-seed abs & Avg. worst-seed norm \\
\midrule
halfcheetah & magnitude & 98\% & 1871.41 & 1.00 \\
 & magnitude\_ste & 30\% & 1908.56 & 1.00 \\
 & random & 85\% & 1844.73 & 0.91 \\
 & saliency & 80\% & 2093.46 & 1.11 \\
hopper & magnitude & 30\% & 1348.83 & 0.92 \\
 & magnitude\_ste & 90\% & 1227.63 & 0.83 \\
 & random & 30\% & 1528.69 & 1.06 \\
 & saliency & 85\% & 1425.76 & 0.96 \\
walker2d & magnitude & 30\% & 1123.91 & 0.53 \\
 & magnitude\_ste & 70\% & 446.48 & 0.22 \\
 & random & 30\% & 2259.16 & 1.11 \\
 & saliency & 30\% & 1403.68 & 0.68 \\
\bottomrule
\end{tabular}
\caption{For each environment and pruning method (SA=on), we summarize robustness by averaging, over the non-clean attacks, the \emph{worst-seed} mean reward \emph{evaluated at the attack-specific sweet spot}, where the sweet spot is defined as the sparsity that maximizes the \emph{average} reward across seeds within the method. We report a single representative pruning percentage per method as the \emph{mode} of the attack-wise sweet spots (ties favor smaller \%). Absolute scores and values normalized to the no-prune (0.30) baseline for each attack are shown.}
\label{tab:avg_worst_seed_per_method_with_percent}
\end{table}

\section{Experimental details and configuration}
\label{app:config}
\textbf{Hyperparameters.} The hyperparameters for PPO are presented in Tables \ref{tab:adv-hparams}

\begin{table}[H]
\centering
\caption{Key PPO and attack-specific hyperparameters for adversarial training in  \texttt{MuJoCo}.}
\label{tab:adv-hparams}
\begin{tabular}{lc}
\toprule
\textbf{Hyperparameter}         & \textbf{Value} \\
\midrule
Total timesteps                 & 50M \\
Learning rate                   & 3e-4 \\
Batch size (envs × steps)       & 2048 × 10 \\
Update epochs                   & 4 \\
Minibatches per update          & 32 \\
$\gamma$ (discount factor)      & 0.99 \\
GAE $\lambda$                   & 0.95 \\
Clipping $\epsilon$             & 0.2 \\
Entropy coefficient             & 0.01 \\
Value function coefficient      & 0.5 \\
Max gradient norm               & 1.0 \\
Adversary hidden size           & 256 \\
Similarity penalty $\lambda_\text{attack}$ & 10 \\
SA Kappa \(\kappa\) & Chosen from \(\kappa \in \{0.3, 0.5, 0.7 \} \) \\
\bottomrule
\end{tabular}
\end{table}

\textbf{Experimental compute resources}\\

All experiments were run as single node jobs across 2 clusters, each comprised of 4 NVIDIA RTX A6000 GPUs - a total of 8 GPUs, each with 48 GB of VRAM.

Upper bounds for compute time are listed below:
\begin{itemize}
    \item Training victim policies - 109 GPU hours
    \item Training adversarial policies - 10 GPU hours
    \item Evaluating adversarial policies against victims - 72 GPU hours
\end{itemize}
An upper bound for total compute time is $109+10+72=191$ GPU hours, or approximately $8$ GPU days.

\noindent\textbf{Policy network architectures.}
We use a unified actor--critic architecture across all \texttt{MuJoCo} tasks. The network is implemented in JAX/Flax and consists of two parallel branches for the actor and critic, sharing the same design principle. Each branch is a two-layer MLP with hidden size 256 and either \texttt{tanh} or \texttt{ReLU} activations (selectable at runtime). Weights are initialized with orthogonal initialization (scaled appropriately), and biases are set to zero.

The actor branch outputs the mean of a diagonal Gaussian distribution over the action space, with fixed covariance $\sigma^2 I$ where $\sigma=0.1$. Together, these define a Multivariate Normal policy distribution. The critic branch outputs a scalar state-value estimate through its own two-layer MLP with the same hidden size and activation.  

This design provides a balanced architecture: compact enough for stable training with pruning, yet expressive enough to capture the dynamics of continuous-control benchmarks.

\section{Adversarial perturbation visualisations}
\label{app:perturbations}

This appendix visualizes the benign observations alongside their corresponding adversarial perturbations for three environments: \texttt{Craftax}, \texttt{HalfCheetah}, and \texttt{Hopper}.



\begin{figure}[H]
    \centering
    \begin{tabular}{cc}
        \toprule
        \textbf{Natural Observation} & \textbf{Adversarial Perturbation} \\
        \midrule
        \includegraphics[width=0.28\textwidth]{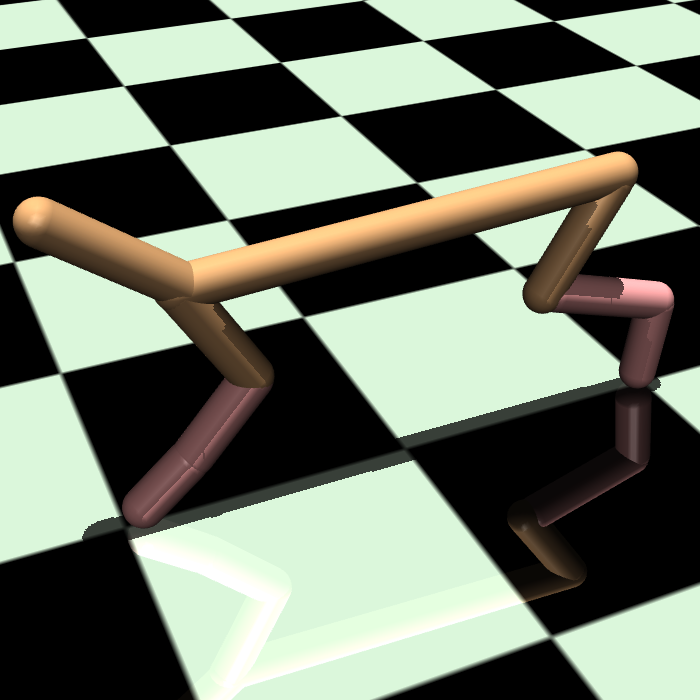} & 
        \includegraphics[width=0.28\textwidth]{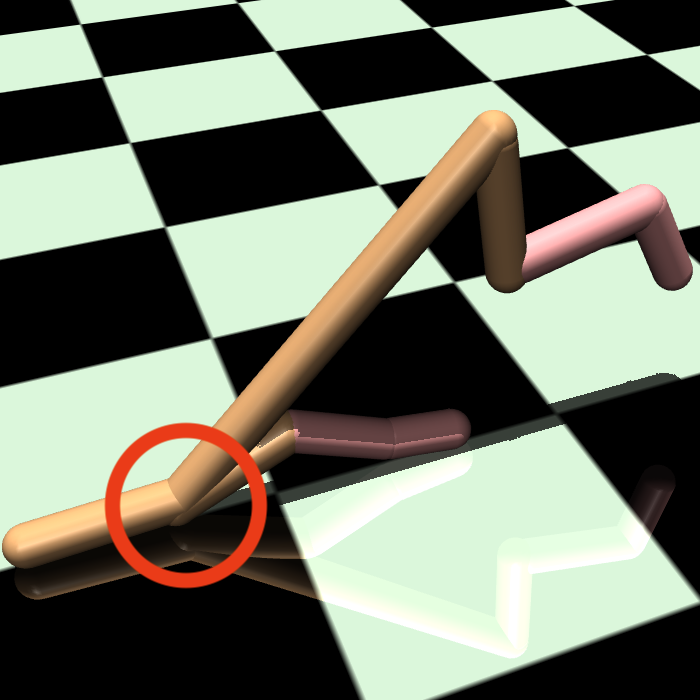} \\
        \includegraphics[width=0.28\textwidth]{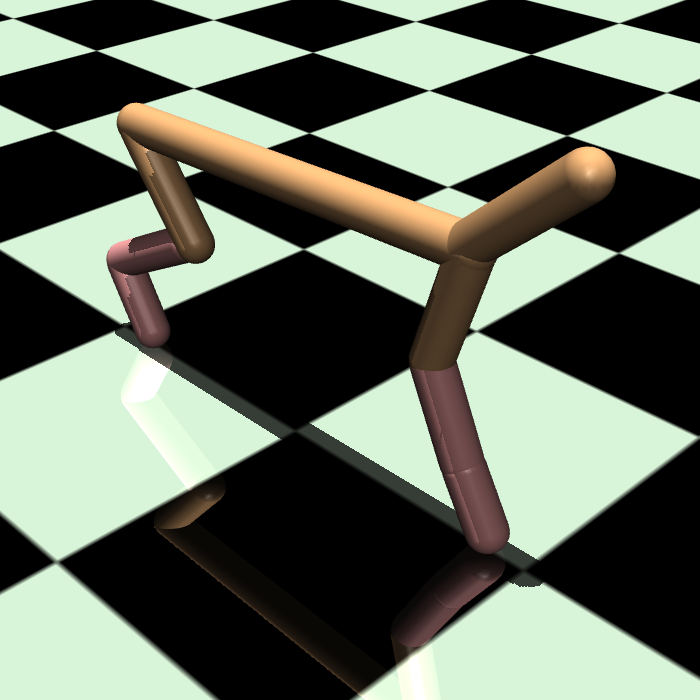} & 
        \includegraphics[width=0.28\textwidth]{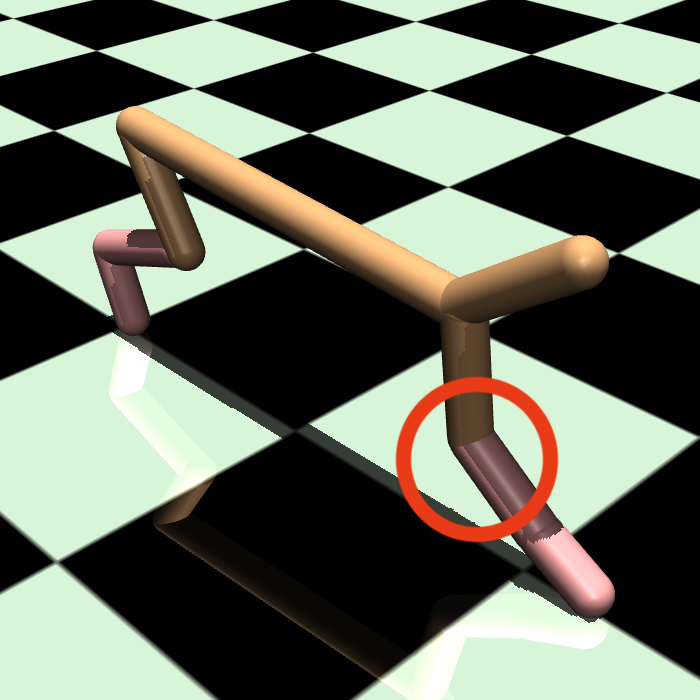} \\
        \bottomrule
    \end{tabular}
    \caption{Comparison of benign observations and their corresponding adversarial perturbations in the \texttt{HalfCheetah} environment. The first row is a particularly severe perturbation, the kind our adversarial framework is \textbf{disincentivized} from producing.}
    \label{fig:halfcheetah_comparison}
\end{figure}

\begin{figure}[H]
    \centering
    \begin{tabular}{cc}
        \toprule
        \textbf{Natural Observation} & \textbf{Adversarial Perturbation} \\
        \midrule
        \includegraphics[width=0.28\textwidth]{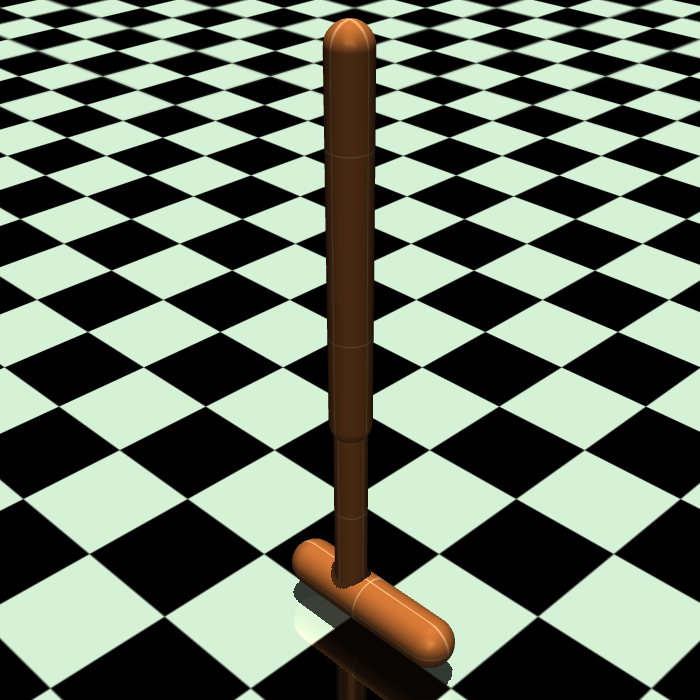} & 
        \includegraphics[width=0.28\textwidth]{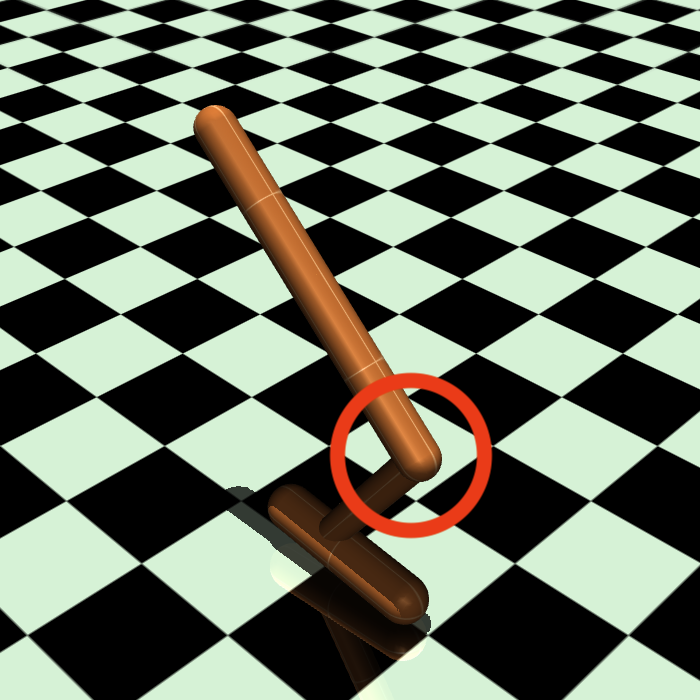} \\
        \includegraphics[width=0.28\textwidth]{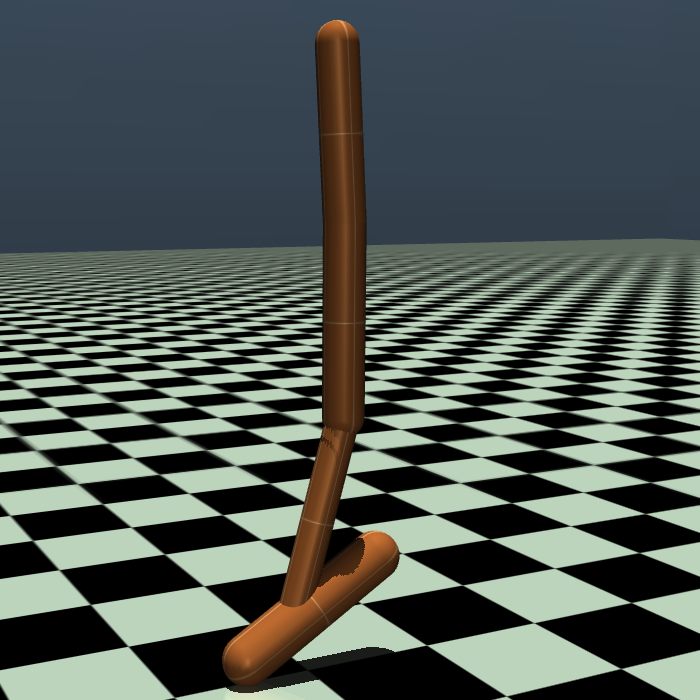} & 
        \includegraphics[width=0.28\textwidth]{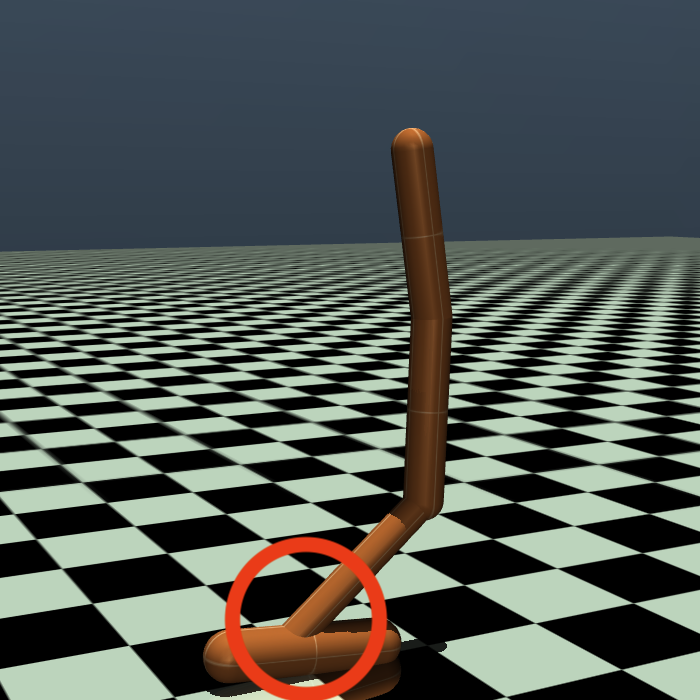} \\
        \bottomrule
    \end{tabular}
    \caption{Comparison of benign observations and their corresponding adversarial perturbations in the \texttt{Hopper} environment.}
    \label{fig:hopper_comparison}
\end{figure}


\section{Reproducibility Statement.} 
We have taken several steps to ensure the reproducibility of our results. 
Theoretical contributions, including proofs of all main theorems and supporting lemmas, are provided in Appendix~\ref{app:proofs}. 
Full algorithmic details, including pseudocode for PPO with pruning and optional SA regularization, are given in Appendix~\ref{app:algo}. 
Experimental settings, including hyperparameters, network architectures, compute resources, and pruning schedules, are described in Appendix~\ref{app:config}. 
Additional empirical results, including robustness--performance trade-offs across environments, per-seed variability, and micro-pruning ablations, are presented in Appendix~\ref{app:figures}. 
Visualizations of adversarial perturbations are included in Appendix~\ref{app:perturbations}. 